\documentclass{article} 
\usepackage{iclr2027_conference,times}

\usepackage{amsmath,amsfonts,bm}

\def\eqref#1{equation~\ref{#1}}

\def\1{\bm{1}}

\DeclareMathAlphabet{\mathsfit}{\encodingdefault}{\sfdefault}{m}{sl}
\SetMathAlphabet{\mathsfit}{bold}{\encodingdefault}{\sfdefault}{bx}{n}

\usepackage{adjustbox}
\usepackage{algorithm2e}
\usepackage{amsfonts}
\usepackage{amsmath}
\usepackage{amssymb}
\usepackage{array}
\usepackage{colortbl}
\usepackage{arydshln}
\usepackage{bbm}
\usepackage{bm}
\usepackage{booktabs}
\usepackage{caption}
\usepackage{CJKutf8}
\usepackage{diagbox}
\usepackage{enumitem}
\usepackage{float}
\usepackage{fontawesome5}
\usepackage{graphicx}
\usepackage{hyperref}
\usepackage{lineno}
\usepackage{listings}
\usepackage{lipsum}
\usepackage{makecell}
\usepackage{marvosym}
\usepackage{microtype}
\usepackage{multicol}
\usepackage{multirow}
\usepackage{pgfplots}
\usepackage{pifont}
\usepackage{subcaption}
\usepackage{tabularx}
\usepackage{tcolorbox}
\usepackage{tikz}
\usepackage[normalem]{ulem}
\usepackage{url}
\usepackage{wrapfig}
\usepackage[table]{xcolor}
\usetikzlibrary{positioning, arrows.meta}
\usetikzlibrary{calc}

\usepackage{xcolor}
\definecolor{myblue}{HTML}{3498DB}
\definecolor{myorange}{HTML}{F39C12}
\definecolor{mygreen}{HTML}{4CBB17}

\usetikzlibrary{3d, arrows.meta, calc, fit, positioning, shapes}
\tcbuselibrary{skins, breakable, theorems}
\pgfplotsset{compat=1.18}

\definecolor{deepred}{rgb}{0.6, 0, 0}
\newcommand{\colorref}[1]{\textcolor{deepred}{\ref{#1}}}
\newcommand{\colorcitep}[1]{\textcolor{deepred}{\citep{#1}}}

\usepackage{amsmath}
\newtheorem{theorem}{Theorem}

\newenvironment{proof}{{\noindent\it Proof.}\quad}{\hfill $\square$\par}

\title{Group-Marginalized Self-Rewarding RL \\Drives Zero-Label Self-Evolving}

\author{\vspace{0.015in}Yiming Wang$^1$, ~Yikang Liu$^1$, ~Qingyuan Tian$^1$, ~Xingyu Chen$^1$,\\
\vspace{0.015in}{\bf Zhuosheng Zhang$^1$, ~Zhaopeng Tu$^2$, ~Rui Wang$^1$}\\
$^1$School of Computer Science, Shanghai Jiao Tong University
~~~~$^2$Tencent\\
\texttt{yiming.wang@sjtu.edu.cn}
}

\iclrfinalcopy 
\begin{document}

\maketitle

\vspace{-0.3in}
\begin{abstract}
Self-rewarding reinforcement learning (RL) enables large language models (LLMs) to self-evolve without human labels.
Existing ensemble-based methods construct reward references from rollout groups and assign rewards accordingly.
However, a response's reward representation also depends on its randomly sampled group context, \emph{i.e.}, the other responses in its group.
Using only one group-context realization may miss desired reward signals and provide unreliable guidance for policy optimization.
To address this issue, we propose \textbf{Group-Marginalized Advantage Estimation (GMAE)}, which aggregates reward realizations across possible contexts into a response-level distribution and estimates expected advantages.
Experiments across eight benchmarks and four base models demonstrate strong performance and cross-domain generalization.
GMAE also exhibits stable learning, low extra cost, and good applicability across training datasets and RL backbones.
\end{abstract}

\begin{figure}[H]
    \vspace{-0.2in}
    \centering
    \includegraphics[width=\linewidth]{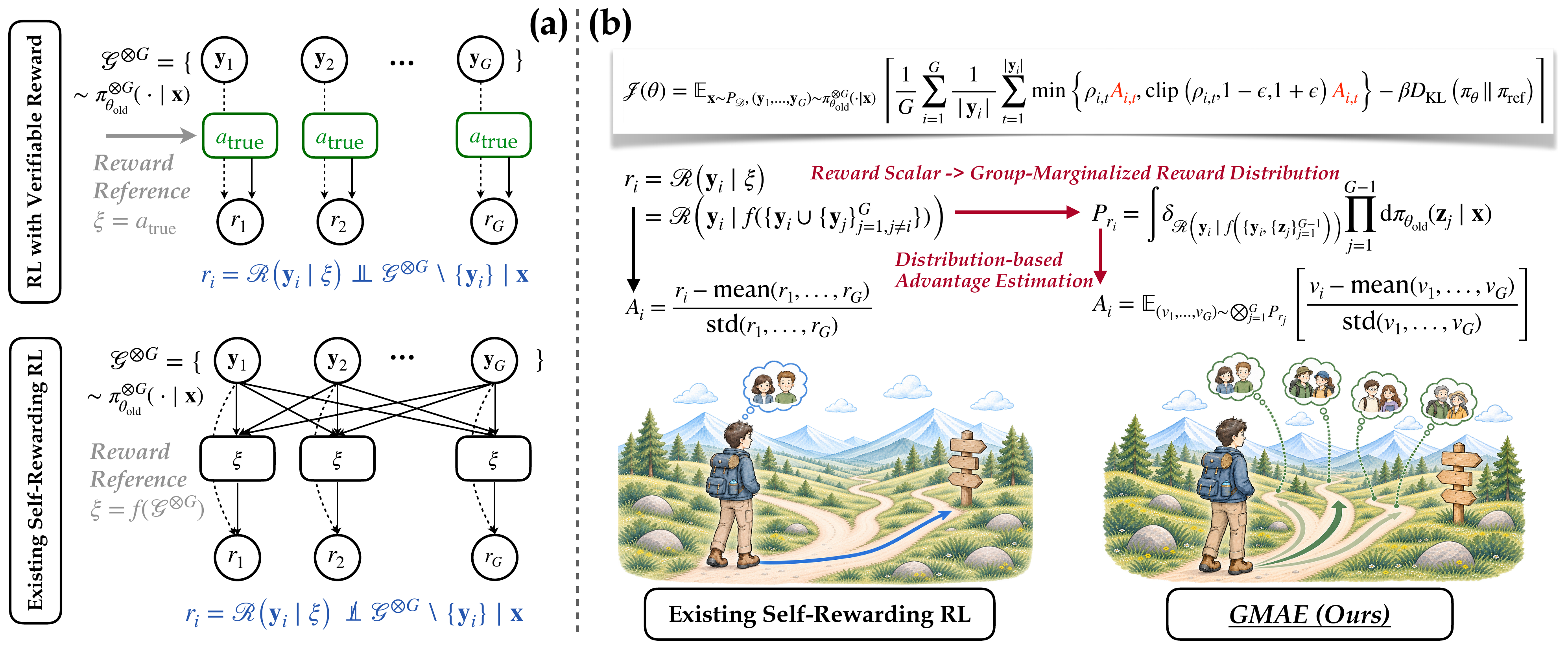}
    \vspace{-0.2in}
    \caption{
    \textbf{(a) GMAE Motivation:}
    RLVR evaluates each response against external ground truth, whereas in self-rewarding RL its reward representation depends on the remaining responses in its rollout group, termed its \textit{group context}.
    \textbf{(b) GMAE Overview:}
    Conventional methods use a single group-context realization to obtain each response's reward representation, while GMAE aggregates rewards across possible contexts into a reward distribution and estimates the expected advantage.
    }
    \label{fig:overview}
\vspace{-0.2in}
\end{figure}

\section{Introduction}
\label{sec:intro}
\vspace{-0.05in}

Self-evolution enables large language models (LLMs) to learn autonomously from their own experience~\colorcitep{tao2024selfevolution}, yet becomes harder as models increasingly decide what, when, and how to evolve~\colorcitep{gao2025selfevolving}.
A fundamental setting is \textit{zero-label} learning, where models derive signals from environments and data without human labels.
In RL post-training, this motivates \textit{self-rewarding reinforcement learning}~\colorcitep{yuan2024selfrewarding,zhao2025learning}, which uses self-generated rewards for policy optimization, unlike \textit{RL with verifiable rewards (RLVR)} relying on ground-truth labels.

Existing self-rewarding RL methods mainly derive supervision from two sources of self-generated information: probability-based and ensemble-based signals~\colorcitep{he2026how}.
Probability-based methods estimate response quality from token probabilities or entropy~\colorcitep{zhao2025learning,li2025confidence}, but often provide weak supervision and may lead to training instability or collapse~\colorcitep{zhang2025no,roy2025you}.
Ensemble-based methods instead aggregate the rollout response group into a reward reference, typically a pseudo-label obtained through majority voting~\colorcitep{wang2023selfconsistency,zuo2026ttrl}, and reward each response against this reference.
Owing to its simplicity and effectiveness, this paradigm is widely adopted, with subsequent methods refining reward-reference construction and reward function~\colorcitep{zhang2026consistent,wang2026selfharmony,wu2026srttrl}.

Despite their promising performance, these ensemble-based methods share a hidden rewarding limitation, as illustrated in \textcolor{deepred}{Figure~\ref{fig:overview}(a)}.
For a response used for policy optimization, its reward representation depends not only on the response itself but also on the other responses in its rollout group.
These responses constitute its randomly sampled \textbf{group context}.
This dependence on a random factor introduces uncertainty into the reward representation: a fixed response may receive the desired reward signal only under certain group contexts, as shown in \textcolor{deepred}{Figure~\ref{fig:reward-stochasticity-violin}}.
However, existing methods represent this context-dependent reward using only a single group-context realization.
If the sampled realization misses the desired signal, that signal cannot guide subsequent policy optimization.
\textit{Thus, relying on a single group-context realization for reward representation makes policy updates sensitive to random rollout sampling, reducing policy optimization reliability.}

Following the above analysis, a reward representation restricted to a single group-context realization captures the desired signal only when that realization happens to be suitable for the response.
Since this depends on both the self-rewarding rule and the response characteristics, it cannot be guaranteed in practice.
Therefore, we consider preserving more reward realizations across different group contexts to reduce the risk of missing desired signals.
This motivates our \textbf{Group-Marginalized Advantage Estimation (GMAE)} strategy, illustrated in \textcolor{deepred}{Figure~\ref{fig:overview}(b)}.
GMAE aggregates these reward realizations into a distribution that serves as the response-level reward representation, then uses it to estimate the response's expected advantage.
Notably, GMAE is {\it a general strategy} agnostic to reward-reference construction and reward function: these components define the self-rewarding rules and serve as the ``fuel'', while GMAE acts as a general ``vehicle'' for leveraging them.
To apply GMAE in practice, we instantiate it with three representative self-rewarding rules.

We conduct extensive experiments with four base models across eight benchmarks covering mathematics, science, knowledge, and coding.
GMAE achieves strong performance and generalizes from mathematical training data to unseen domains, demonstrating more reliable reward estimation.
It also exhibits stable learning dynamics and low variance across random seeds, demonstrating more stable optimization.
Further experiments also confirm its applicability across training datasets and RL backbones.
Overall, these results establish GMAE as a reliable general framework for zero-label self-evolution through self-rewarding RL.

\vspace{-0.05in}
\section{Preliminaries: RLVR and Self-Rewarding RL}
\label{sec:preliminary}

In modern LLM post-training, RLVR is a common RL paradigm that improves reasoning through verifiable rewards.
As a representative algorithm, GRPO~\colorcitep{shao2024deepseekmath} uses a PPO-style objective~\colorcitep{schulman2017proximal,ouyang2022training}, removing the value model and deriving advantages from group-relative rewards.
At each training step, given a prompt \(\mathbf{x}\) sampled from the dataset \(\mathcal{D}\), GRPO samples \(G\) responses
\(\mathcal{G}^{\otimes G}=\{\mathbf{y}_1,\dots,\mathbf{y}_G\}\)
from
\(\pi_{\theta_{\mathrm{old}}}(\cdot\mid\mathbf{x})\).
Let \(\mathcal A(\mathbf y_i)\) denote the answer corresponding to response \(\mathbf y_i\).
Each response receives a scalar reward \(r_i\), forming
\(\mathbf r=(r_1,\ldots,r_G)\), and its advantage is obtained by group-wise normalization:

\vspace{-0.15in}
\begin{equation}
A_i=\frac{r_i-\mu_{\mathbf r}}{\sigma_{\mathbf r}},
\qquad
\mu_{\mathbf r}=\frac{1}{G}\sum_{j=1}^{G}r_j,
\qquad
\sigma_{\mathbf r}=\sqrt{\frac{1}{G}\sum_{j=1}^{G}\left(r_j-\mu_{\mathbf r}\right)^2}.
\label{eq:grpo-advantage}
\end{equation}
\vspace{-0.15in}

Throughout, zero advantages are assigned whenever \(\sigma_{\mathbf r}=0\).
The response-level advantage is shared across all tokens, \emph{i.e.},
\(A_{i,t}=A_i\).
Defining
$
\rho_{i,t}
=
\pi_{\theta}(y_{i,t}\mid\mathbf{x},\mathbf{y}_{i,<t})
/
\pi_{\theta_{\mathrm{old}}}(y_{i,t}\mid\mathbf{x},\mathbf{y}_{i,<t}),
$
GRPO optimizes a clipped objective with KL regularization \(D_{\mathrm{KL}}\):

\vspace{-0.15in}
\begin{equation}
\resizebox{0.92\columnwidth}{!}{$
\displaystyle
\mathcal{J}(\theta)
=
\mathbb{E}_{\mathbf{x}\sim P_{\mathcal D},\,\mathcal{G}^{\otimes G}\sim\pi_{\theta_{\mathrm{old}}}^{\otimes G}(\cdot\mid\mathbf{x})}
\left[
\frac{1}{G}\sum_{i=1}^{G}\frac{1}{|\mathbf y_i|}\sum_{t=1}^{|\mathbf y_i|}
\min\!\left\{
\rho_{i,t}A_{i,t},
\operatorname{clip}\!\left(\rho_{i,t},1-\epsilon,1+\epsilon\right)A_{i,t}
\right\}
-\beta D_{\mathrm{KL}}\!\left(\pi_\theta\,\|\,\pi_{\mathrm{ref}}\right)
\right].
$}
\label{eq:grpo-objective}
\end{equation}
\vspace{-0.1in}

\vspace{-0.1in}
\paragraph{Standard RLVR.}
In \textcolor{deepred}{Eq.\ref{eq:grpo-advantage}}, the reward \(r_i\) of each response \(\mathbf y_i\) is computed by a reward function \(\mathcal R(\mathbf y_i\mid\xi)\), where \(\xi\) denotes the \textit{reward reference}.
In standard RLVR, this reference is given by the ground-truth answer \(a_{\mathrm{true}}\).
The reward is therefore determined solely by the fixed external reference:
\begin{equation}
    \xi
    \equiv
    a_{\mathrm{true}},
    \qquad
    r_i^{\mathrm{VR}}
    =
    \mathcal R\!\left(
        \mathbf y_i
        \mid
        \xi
    \right)
    =
    \mathbbm{1}\!\left[
        \mathcal A(\mathbf y_i)=a_{\mathrm{true}}
    \right].
    \label{eq:rlvr-reward}
\end{equation}

\paragraph{Mainstream Pipeline of Self-Rewarding RL.}
However, \(a_{\mathrm{true}}\) is unavailable in self-rewarding RL, so \(\xi\) must be inferred from the sampled rollout group
\(\mathcal G^{\otimes G}=\{\mathbf y_1,\ldots,\mathbf y_G\}\).
Specifically, a group-dependent function \(f\) constructs the reward reference
\(\xi=f(\mathcal G^{\otimes G})\), against which the reward for \(\mathbf y_i\) is:

\vspace{-0.3in}
\begin{equation}
    r_i^{\mathrm{SR}}
    =
    \mathcal R\!\left(
        \mathbf y_i
        \mid
        \xi
    \right)
    =
    \mathcal R\!\left(
        \mathbf y_i
        \mid
        f(\mathcal G^{\otimes G})
    \right).
    \label{eq:ttrl-reward}
\end{equation}
A common choice is majority voting~\colorcitep{wang2023selfconsistency,zuo2026ttrl}, where \(\xi\) is the most frequent answer among
\(\{\mathcal A(\mathbf y_1),\ldots,\mathcal A(\mathbf y_G)\}\), and \(\mathcal R\) is a binary indicator.
Other methods refine \(f\) or \(\mathcal R\) while following the same pipeline.
Their details are provided in
\textcolor{deepred}{Appendix~\ref{app:baseline-methods}}.

\section{On the Limitation of Self-Reward Uncertainty}
\label{sec:existing-limitations}

Existing self-rewarding methods differ in their choice of $(f,\mathcal R)$, but share a common property: the reward assigned to each response depends on the rollout group it was sampled from.
This dependence is implicit in \textcolor{deepred}{Eq.\ref{eq:ttrl-reward}}, where the pseudo-label
$f(\mathcal G^{\otimes G})$
is constructed from the entire rollout group.

To make this dependence explicit, we fix a response $\mathbf y_i$ and decompose its rollout group as
\begin{equation}
    \mathcal G^{\otimes G}
    =
    \{\mathbf y_i\}
    \cup
    \mathcal H_i,
    \qquad
    \mathcal H_i
    \triangleq
    \{\mathbf y_j\}_{j\neq i},
    \label{eq:group-context}
\end{equation}
where $\mathcal H_i$ denotes the \textbf{group context} realized alongside $\mathbf y_i$.
Since all responses are independently sampled from the same policy $\pi_{\theta_{\mathrm{old}}}(\cdot\mid\mathbf x)$, group contexts follow the same underlying distribution for all focal responses.
We therefore use $\mathcal H$ to denote a generic random group context.
For a fixed response $\mathbf y_i$, its reward under $\mathcal H$ is
\begin{equation}
\resizebox{0.85\columnwidth}{!}{$
    r_i(\mathcal H)
    =
    \mathcal R\!\left(
        \mathbf y_i
        \mid
        f\!\left(
            \{\mathbf y_i\}\cup\mathcal H
        \right)
    \right),
    \qquad
    \mathcal H=\{\mathbf z_1,\ldots,\mathbf z_{G-1}\}
    \sim
    Q(\cdot\mid\mathbf x)
    \equiv
    \pi_{\theta_{\mathrm{old}}}^{\otimes(G-1)}(\cdot\mid\mathbf x)
$}
    \label{eq:context-dependent-reward}
\end{equation}
Although $r_i(\mathcal H_i)$ is deterministic once the realized context
$\mathcal H_i$ is fixed, $\mathcal H$ itself is randomly sampled.
Thus, the reward intended to estimate the quality of $\mathbf y_i$ depends not only on the response itself but also on its group context, introducing uncertainty into the resulting learning signal.

\begin{figure}[t]
    \vspace{-0.2in}
    \centering
    \includegraphics[width=\textwidth]{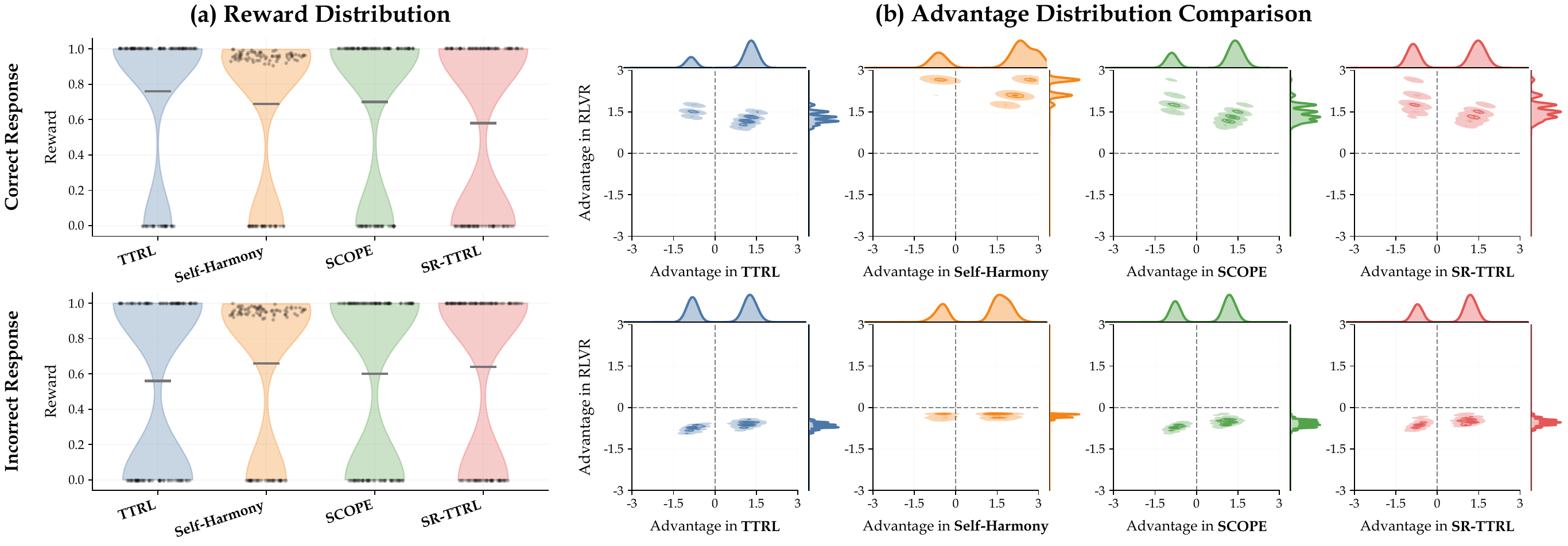}
    \vspace{-0.25in}
    \caption{
    \textbf{Self-Reward Uncertainty Demonstration.}
    \textbf{(a)} Reward distributions for a fixed correct response (top) and a fixed incorrect response (bottom) across group contexts under TTRL \colorcitep{zuo2026ttrl}, Self-Harmony \colorcitep{wang2026selfharmony}, SCOPE \colorcitep{wang2026beyond}, and SR-TTRL \colorcitep{wu2026srttrl}. Dots denote individual realizations and gray bars denote means.
    \textbf{(b)} Joint distributions of self-rewarding ($x$-axis) and RLVR ($y$-axis) advantages for the same responses.
    Results here are from \texttt{Qwen3-8B-Base} on a \texttt{DAPO-14K} prompt; more cases are provided in \textcolor{deepred}{Appendix~\ref{app:more-motivation}}.
    }
    \label{fig:reward-stochasticity-violin}
    \vspace{-0.2in}
\end{figure}

\vspace{-0.1in}
\paragraph{Empirical Validation.}
We examine how group-context variation affects the reward and advantage of the same response.
Given a model $\pi$, we fix a prompt $\mathbf x$ and sample one correct focal response $\mathbf y^{+}$ and one incorrect focal response $\mathbf y^{-}$.
For each response, we sample $M=100$ independent group contexts, each containing $G-1$ newly sampled responses with $G=16$.
These contexts produce $M$ self-reward realizations and group-normalized advantages, forming their empirical distributions.
For comparison, we compute the corresponding RLVR advantage under each context.

\textcolor{deepred}{Figure~\colorref{fig:reward-stochasticity-violin}} presents results for four representative self-rewarding methods.
As shown in \textcolor{deepred}{\textbf{(a)}}, the reward of the same response varies substantially across group contexts.
Although the desired signal is positive for a correct response and negative for an incorrect one, both may receive the opposite signal under some contexts.
\textit{Thus, restricting reward assignment to a single group-context realization may miss the desired signal, limiting the reliability of self-reward as an estimate of response quality.}
In \textcolor{deepred}{\textbf{(b)}}, RLVR advantages vary in magnitude but preserve the desired direction: correct responses remain non-negative, whereas incorrect ones remain non-positive.
By contrast, self-rewarding advantages often change sign, causing the same response to be reinforced under one context but suppressed under another.
\textit{Therefore, group-context uncertainty in reward assignment can propagate into advantage estimation, making the optimization direction dependent on the sampled context.}
\section{GMAE: Group-Marginalized Advantage Estimation}
\label{sec:gmae}


\textcolor{deepred}{Section~\ref{sec:existing-limitations}} shows that a correct response may miss a positive reward signal under a single group context. However, the suitable context is response-dependent and unknown in advance, making it difficult to identify in practice. Therefore, instead of relying on one group context, we aggregate reward realizations across contexts into a distribution-level response representation, improving the chance of retaining desired signals. This idea motivates our \textbf{Group-Marginalized Advantage Estimation (GMAE)}. In \textcolor{deepred}{Section~\ref{sec:gmae-main}}, we will introduce its core algorithm, and in \textcolor{deepred}{Section~\ref{sec:gmae-instantiations}}, we will instantiate it with different self-rewarding rules.

\vspace{-0.05in}
\subsection{Core Algorithm}
\label{sec:gmae-main}

\paragraph{\textcolor{deepred}{Main Step I:} Group-Marginalized Reward Distribution.}
For a response $\mathbf y_i$, instead of using the scalar reward induced by the single realized group context $\mathcal G^{\otimes G} / \{\mathbf y_i\}$, we consider its reward over the group-context distribution in \textcolor{deepred}{Eq.\ref{eq:context-dependent-reward}}.
Let $\delta_v$ denote the Dirac measure concentrated at $v$, we define the \textit{group-marginalized reward distribution} as

\vspace{-0.15in}
\begin{equation}
\resizebox{0.8\columnwidth}{!}{$
\displaystyle
\textcolor{myblue}{
P_{r_i} =
\int
\delta_{\mathcal R\left(\mathbf y_i\mid f\left(\{\mathbf y_i\}\cup\mathcal H\right)\right)}
\,\mathrm dQ(\mathcal H\mid\mathbf x)
}
=
\int
\delta_{\mathcal R\left(\mathbf y_i\mid f\left(\{\mathbf y_i,\mathbf z_1,\ldots,\mathbf z_{G-1}\}\right)\right)}
\prod_{j=1}^{G-1}\mathrm d\pi_{\theta_{\mathrm{old}}}(\mathbf z_j\mid\mathbf x),
$}
\label{eq:group-marginalized-reward}
\end{equation}
\vspace{-0.15in}

However, directly computing \textcolor{deepred}{Eq.\ref{eq:group-marginalized-reward}} is intractable because it requires marginalizing over the $(G-1)$-fold product distribution.
Naive Monte Carlo estimation can approximate this marginalization, but requires repeatedly sampling fresh group contexts for each response, incurring substantial rollout cost.
Therefore, we reuse sampled responses through a shared candidate pool~\colorcitep{wang2026breaking}.

Specifically, besides the $G$ main responses
$\mathcal G^{\otimes G}=\{\mathbf y_1,\ldots,\mathbf y_G\}$ used for policy optimization,
we sample $G'$ auxiliary responses to form
$\mathcal C=\{\mathbf y_1,\ldots,\mathbf y_{G+G'}\}$.
For each main response $\mathbf y_i$ ($1 \leq i \leq G$), choosing $G-1$ companions from
$\mathcal C\setminus\{\mathbf y_i\}$ yields
$\binom{G+G'-1}{G-1}$ possible group contexts.
We uniformly sample $K$ distinct contexts
$\{\mathcal S_i^k\}_{k=1}^{K}$ up to
$\tau$,
and construct the \textit{empirical distribution}\footnote{To avoid combinatorial explosion in enumeration, we set $\tau=10^4$ by increasing the number of evaluated contexts tenfold from $10$ and observing negligible changes in the empirical reward distribution beyond $10^4$.}

\vspace{-0.15in}
\begin{equation}
\resizebox{0.65\columnwidth}{!}{$
\displaystyle
\textcolor{mygreen}{
\widehat P_{r_i}
=
\frac{1}{K}\sum_{k=1}^{K}\delta_{\mathcal{R}\left( \mathbf y_i \mid f\left(\{\mathbf y_i\}\cup\mathcal S_i^k\right) \right)}
},
\qquad
K=\min\left\{\tau,\binom{G+G'-1}{G-1}\right\}.
$}
\label{eq:gmae-empirical-reward}
\end{equation}
\vspace{-0.1in}

Now we measure the error from replacing the ideal $P_{r_i}$ in \textcolor{deepred}{Eq.\ref{eq:group-marginalized-reward}} with the empirical $\widehat P_{r_i}$ in \textcolor{deepred}{Eq.\ref{eq:gmae-empirical-reward}} with the total variation distance $d_{\mathrm{TV}}(\cdot, \cdot)$.
For a finite reward-value space $\mathcal V_i=\operatorname{supp}(P_{r_i})$, with the confidence level and $G$ fixed, the error scales as

\vspace{-0.25in}
\begin{equation}
\resizebox{0.45\columnwidth}{!}{$
\displaystyle
d_{\mathrm{TV}}\!\left(\widehat P_{r_i},P_{r_i}\right)
=
\mathcal O\!\left(
\sqrt{\frac{|\mathcal V_i|}{K}}
+
\sqrt{\frac{|\mathcal V_i|}{G+G'}}
\right).
$}
\label{eq:bound}
\end{equation}
\vspace{-0.15in}

It decreases with $K$ and $G'$ but increases with $|\mathcal V_i|$.
The exact bound and proof are in
\textcolor{deepred}{Appendix~\ref{app:gmae-approximation-bound}}.

\vspace{-0.1in}
\paragraph{\textcolor{deepred}{Main Step II:} Distribution-based Advantage Estimation.}
After marginalizing over group contexts, each $\widehat P_{r_i}$ defines a response-level reward profile that is no longer tied to any particular group-context realization.
Accordingly, GMAE treats these distributions as response-level reward representations and independently draws reward realizations from them for advantage estimation.

Specifically, we construct the product distribution $\bigotimes_{j=1}^{G}\widehat P_{r_j}$ over joint reward realizations $\mathbf v=(v_1,\ldots,v_G)$,
where each $v_j\sim\widehat P_{r_j}$ and
$\mathbf v\in\widehat{\mathcal V}
=\prod_{j=1}^{G}\widehat{\mathcal V}_j$.
For each $\mathbf v$, we compute the group-normalized advantage and then take its expectation under the product distribution, yielding

\vspace{-0.15in}
\begin{equation}
\textcolor{mygreen}{
\resizebox{0.85\columnwidth}{!}{$
\displaystyle
A_i^{\mathrm{gmae}}
=
\mathbb E_{\mathbf v\sim\bigotimes_{j=1}^{G}\widehat P_{r_j}}
\left[
\frac{v_i-\mu_{\mathbf v}}{\sigma_{\mathbf v}}
\right]
=
\sum_{\mathbf v\in\widehat{\mathcal V}}
\left\{
\prod_{j=1}^{G}
\left[
\frac{1}{K}\sum_{k=1}^{K}
\mathbbm{1}
\left[
\mathcal R\left(
\mathbf y_j
\mid
f\left(\{\mathbf y_j\}\cup\mathcal S_j^k\right)
\right)
=
v_j
\right]
\right]
\right\}
\frac{v_i-\mu_{\mathbf v}}{\sigma_{\mathbf v}}.
$}
}
\label{eq:gmae-expected-advantage}
\end{equation}

\vspace{-0.1in}
\paragraph{\textcolor{deepred}{Attached Mechanism:} Advantage Scale Calibration.}
Although group-marginalized advantages preserve richer reward information, {\it marginalization can shrink their scale and weaken policy updates}.
Since $\sum_{i=1}^{G}A_i^{\mathrm{gmae}}=0$, their variance is
$\operatorname{Var}(\{A_i^{\mathrm{gmae}}\}_{i=1}^{G})
=\frac{1}{G}\sum_{i=1}^{G}(A_i^{\mathrm{gmae}})^2$,
which directly reflects the advantage scale and effective update strength.
By Jensen's inequality~\colorcitep{jensen1906},

\vspace{-0.1in}
\begin{equation}
\resizebox{0.7\columnwidth}{!}{$
\displaystyle
\operatorname{Var}\!\left(\{A_i^{\mathrm{gmae}}\}_{i=1}^{G}\right)
=
\frac{1}{G}\sum_{i=1}^{G}
\left(
\mathbb E_{\mathbf v}
\left[
\frac{v_i-\mu_{\mathbf v}}{\sigma_{\mathbf v}}
\right]
\right)^2
\le
\mathbb E_{\mathbf v}
\left[
\frac{1}{G}\sum_{i=1}^{G}
\left(
\frac{v_i-\mu_{\mathbf v}}{\sigma_{\mathbf v}}
\right)^2
\right]
\le 1.
$}
\label{eq:gmae-contraction}
\end{equation}
\vspace{-0.1in}

The right-hand side is the expected variance of conventional group-wise advantages without marginalization and is at most 1 under normalization.
Thus, marginalization often reduces the advantage scale, {\it motivating calibration to restore the non-marginalized update strength}.


Prompt-level calibration could counteract GMAE's reduced advantage scale by matching each prompt to its scale from a single group context. However, scale differences across prompts can reflect prompt-specific reward uncertainty.
A response favored in some contexts but disfavored in others has normalized advantages that partially cancel under marginalization. This reduces the prompt's scale and makes its updates more conservative. Calibrating each prompt separately would erase these uncertainty-dependent differences. We formalize this relation in \textcolor{deepred}{Appendix~\ref{app:advantage-uncertainty}}.

Therefore, we adopt batch-level calibration to restore the overall scale while preserving relative scale differences across prompts.
For a batch of $B$ prompts, let
$A_{b,i}^{\mathrm{base}}$ denote the non-marginalized advantage.
We denote the batch-level scales of
$A_{b,i}^{\mathrm{base}}$ and $A_{b,i}^{\mathrm{gmae}}$
by $\sigma_{\mathrm{base}}$ and $\sigma_{\mathrm{gmae}}$:

\vspace{-0.1in}
\begin{equation}
\resizebox{0.7\columnwidth}{!}{$
\displaystyle
\sigma_{\mathrm{base}}
=
\sqrt{
\frac{1}{BG}
\sum_{b=1}^{B}\sum_{i=1}^{G}
\left(A_{b,i}^{\mathrm{base}}\right)^2
},
\qquad
\sigma_{\mathrm{gmae}}
=
\sqrt{
\frac{1}{BG}
\sum_{b=1}^{B}\sum_{i=1}^{G}
\left(A_{b,i}^{\mathrm{gmae}}\right)^2
}.
$}
\label{eq:gmae-batch-scale}
\end{equation}
\vspace{-0.1in}

We use $\sigma_{\mathrm{base}}$ as the reference because it retains the conventional update scale induced by group-wise normalization.
We then calibrate all marginalized advantages using a shared factor $\lambda$, with a lower bound of $1$ to ensure that calibration only compensates for scale contraction:

\vspace{-0.1in}
\begin{equation}
\textcolor{mygreen}{
\widetilde A_{b,i}^{\mathrm{gmae}}
=
\lambda A_{b,i}^{\mathrm{gmae}}
},
\qquad
\lambda
=
\max\left(
1,\,
\frac{\sigma_{\mathrm{base}}}{\sigma_{\mathrm{gmae}}}
\right).
\label{eq:gmae-calibration}
\end{equation}
\vspace{-0.3in}

\subsection{Instantiating GMAE with Self-Rewarding Rules}
\label{sec:gmae-instantiations}

As shown in \textcolor{deepred}{Section~\ref{sec:gmae-main}}, GMAE is a general strategy agnostic to the reward-reference construction $f$ and reward function $\mathcal R$.
Different $(f,\mathcal R)$ configurations define different self-rewarding rules that provide the ``fuel'', while GMAE serves as the ``vehicle'' that better leverages these signals for policy optimization.
Naturally, concrete choices of $f$ and $\mathcal R$ are required to instantiate GMAE.

We first instantiate $\mathcal R$ with the binary reward
$\mathcal R_{\mathrm{bin}}(\mathbf y_i\mid\xi)
=\mathbbm{1}[\mathcal A(\mathbf y_i)=\xi]$,
where $\xi$ denotes the reward reference.
We adopt it for three reasons.
First, it follows the standard reward design in RLVR.
Second, it restricts $\mathcal V_i,\widehat{\mathcal V}_i\subseteq\{0,1\}$, and hence $|\widehat{\mathcal V}|\leq2^G$, minimizing the reward-value space enumerated in \textcolor{deepred}{Eq.\ref{eq:gmae-expected-advantage}}.
Third, $|\mathcal V_i|\leq2$ tightens the approximation bound in \textcolor{deepred}{Eq.\ref{eq:bound}}.
With $\mathcal R_{\mathrm{bin}}$ fixed, we combine it with three representative $f$, yielding three corresponding GMAE instantiations:

\vspace{-0.1in}
\paragraph{GMAE$_0$: Standard Voting.}
Following TTRL~\colorcitep{zuo2026ttrl}, GMAE$_0$ constructs the pseudo-label by majority voting~\colorcitep{wang2023selfconsistency} over responses sampled from the current policy:

\vspace{-0.2in}
\begin{equation}
\xi_0
=
f_0(\mathcal G^{\otimes G})
=
\arg\max_{a}
\sum_{j=1}^{G}
\mathbbm{1}\!\left[
\mathcal A(\mathbf y_j)=a
\right].
\label{eq:gmae-zero}
\end{equation}
\vspace{-0.1in}

\vspace{-0.15in}
\paragraph{GMAE$_{\mathrm{CR}}$: Co-Refinement.}
Following Co-Reward~\colorcitep{zhang2026corewarding}, GMAE$_{\mathrm{CR}}$ defines the group context $\mathcal H$ as $G$ responses sampled from a slowly updated co-teacher and constructs the reward reference by majority voting over this context, where $\alpha_t$ denotes the teacher retention coefficient:

\vspace{-0.2in}
\begin{equation}
\resizebox{0.92\columnwidth}{!}{$
\displaystyle
\widetilde{\theta}_t
=
\alpha_t\widetilde{\theta}_{t-1}
+
(1-\alpha_t)\theta_{\mathrm{old},t},
\quad
\widetilde{\mathcal H}
=
\{\widetilde{\mathbf y}_j\}_{j=1}^{G}
\sim
\pi_{\widetilde{\theta}_t}^{\otimes G}(\cdot\mid\mathbf x),
\quad
f_{\mathrm{CR}}\!\left(\widetilde{\mathcal H}\right)
=
\arg\max_a
\sum_{j=1}^{G}
\mathbbm{1}\!\left[
\mathcal A(\widetilde{\mathbf y}_j)=a
\right].
$}
\label{eq:gmae-cr}
\end{equation}
\vspace{-0.1in}

\vspace{-0.15in}
\paragraph{GMAE$_{\mathrm{SR}}$: Self-Refinement.}
Following SR-TTRL~\colorcitep{wu2026srttrl}, GMAE$_{\mathrm{SR}}$ uses self-reflective verification instead of voting.
For each distinct answer $a$, it summarizes a representative response and constructs the reward reference through self-reflective comparison:

\vspace{-0.2in}
\begin{equation}
\begin{aligned}
&\mathbf s_a
=
\operatorname{Summ}_{\pi_{\theta_{\mathrm{old}}}}
\left(
\mathbf x,
\operatorname{Rep}
\left\{
\mathbf y_j
\mid
\mathcal A(\mathbf y_j)=a
\right\}
\right),
\qquad
a\in\mathcal U,
\qquad
\mathcal U
=
\{\mathcal A(\mathbf y_j)\}_{j=1}^{G}.
\\
&\xi_{\mathrm{SR}}
=
f_{\mathrm{SR}}(\mathcal G^{\otimes G})
=
\operatorname{Reflect}_{\pi_{\theta_{\mathrm{old}}}}
\left(
\mathbf x,
\{(a,\mathbf s_a)\}_{a\in\mathcal U}
\right).
\end{aligned}
\label{eq:gmae-sr}
\end{equation}
\vspace{-0.2in}

Since our core contribution lies in the general GMAE strategy, we only briefly introduce the self-rewarding rules used in these three instantiations here.
Their detailed {\bf pseudocode} and {\bf computational complexity analyses} are deferred to \textcolor{deepred}{Appendix~\ref{app:gmae-instantiations}}.
\vspace{-0.05in}
\section{Experiments}

\newcommand{\stdv}[2]{#1\ensuremath{_{\scriptscriptstyle #2}}}

\begin{table*}[t]
\vspace{-0.in}
\caption{
\textbf{Main Results Across Models and Domains:}
\textsc{Mean@16} scores on eight benchmarks across four base models.
\textit{Average} is computed over all eight benchmarks.
{\bf Bold} and \underline{underlined} values denote the {\bf best} and \underline{second-best} results among all self-rewarding methods, respectively.
}
\vspace{-0.1in}
\centering
\footnotesize
\renewcommand\arraystretch{1.05}
\setlength{\tabcolsep}{1.mm}

\resizebox{\textwidth}{!}{
\begin{tabular}{lccccccccc@{\hspace{4mm}}ccccccccc}

\toprule

\multirow{2}{*}{\textbf{Method}}
& \multicolumn{5}{c}{\textbf{Mathematics}}
& \textbf{Science}
& \textbf{Knowledge}
& \textbf{Coding}
& \multirow{2}{*}{\textbf{\textit{Average}}}
&
\multicolumn{5}{c}{\textbf{Mathematics}}
& \textbf{Science}
& \textbf{Knowledge}
& \textbf{Coding}
& \multirow{2}{*}{\textbf{\textit{Average}}}
\\

\cmidrule(lr){2-6}
\cmidrule(lr){7-7}
\cmidrule(lr){8-8}
\cmidrule(lr){9-9}
\cmidrule(lr){11-15}
\cmidrule(lr){16-16}
\cmidrule(lr){17-17}
\cmidrule(lr){18-18}

& {\bf MATH500}
& {\bf AMC}
& {\bf AIME24}
& {\bf AIME25}
& {\bf AIME26}
& {\bf GPQA}
& {\bf MMLU-Pro}
& {\bf LiveCode}
&
& {\bf MATH500}
& {\bf AMC}
& {\bf AIME24}
& {\bf AIME25}
& {\bf AIME26}
& {\bf GPQA}
& {\bf MMLU-Pro}
& {\bf LiveCode}
\\

\cmidrule(lr){2-19}

& \multicolumn{9}{c}{\texttt{\textbf{Qwen3-1.7B-Base}}}
& \multicolumn{9}{c}{\texttt{\textbf{Qwen3-4B-Base}}}
\\

\midrule

Raw Model
& \stdv{45.7}{0.22} & \stdv{17.0}{0.09} & \stdv{2.3}{0.07} & \stdv{0.8}{0.18} & \stdv{1.3}{0.19} & \stdv{7.9}{0.23} & \stdv{13.3}{0.29} & \stdv{6.7}{0.05} & \stdv{11.9}{0.07}
& \stdv{50.5}{0.06} & \stdv{24.3}{0.18} & \stdv{8.0}{0.20} & \stdv{3.8}{0.11} & \stdv{5.4}{0.05} & \stdv{15.8}{0.27} & \stdv{14.2}{0.17} & \stdv{15.3}{0.06} & \stdv{17.2}{0.09} \\

{\it w/ Verifiable Reward}
& \stdv{{\it 64.6}}{0.08} & \stdv{{\it 27.1}}{0.13} & \stdv{{\it 7.1}}{0.20} & \stdv{{\it 5.4}}{0.18} & \stdv{{\it 4.2}}{0.23} & \stdv{{\it 22.6}}{0.24} & \stdv{{\it 31.0}}{0.23} & \stdv{{\it 14.2}}{0.07} & \stdv{{\it 22.0}}{0.17}
& \stdv{{\it 81.3}}{0.16} & \stdv{{\it 48.0}}{0.25} & \stdv{{\it 19.3}}{0.03} & \stdv{{\it 17.7}}{0.11} & \stdv{{\it 14.8}}{0.18} & \stdv{{\it 33.0}}{0.17} & \stdv{{\it 39.4}}{0.32} & \stdv{{\it 19.1}}{0.08} & \stdv{{\it 34.1}}{0.13} \\

\hdashline

Intuitor$^*$
& \stdv{52.1}{0.14} & \stdv{19.8}{0.10} & \stdv{1.1}{0.24} & \stdv{1.2}{0.13} & \stdv{0.8}{0.18} & \stdv{13.3}{0.24} & \stdv{15.7}{0.30} & \stdv{10.7}{0.22} & \stdv{14.3}{0.21}
& \stdv{67.1}{0.18} & \stdv{32.9}{0.04} & \stdv{9.2}{0.21} & \stdv{2.7}{0.20} & \stdv{0.7}{0.20} & \stdv{30.0}{0.22} & \stdv{41.1}{0.23} & \stdv{15.8}{0.08} & \stdv{24.9}{0.21} \\

EM-RL$^*$
& \stdv{57.5}{0.08} & \stdv{19.2}{0.23} & \stdv{2.6}{0.18} & \stdv{1.7}{0.13} & \stdv{0.9}{0.22} & \stdv{13.2}{0.32} & \stdv{14.8}{0.19} & \stdv{11.6}{0.08} & \stdv{15.2}{0.06}
& \stdv{70.9}{0.19} & \stdv{30.5}{0.18} & \stdv{4.4}{0.09} & \stdv{6.3}{0.04} & \stdv{2.6}{0.08} & \stdv{28.3}{0.19} & \stdv{40.8}{0.25} & \stdv{16.3}{0.15} & \stdv{25.0}{0.21} \\

TTRL
& \stdv{61.1}{0.13} & \stdv{24.0}{0.22} & \stdv{3.4}{0.24} & \stdv{2.7}{0.07} & \stdv{1.1}{0.04} & \stdv{20.2}{0.22} & \stdv{23.7}{0.22} & \stdv{11.3}{0.24} & \stdv{18.4}{0.15}
& \stdv{74.6}{0.23} & \stdv{37.8}{0.04} & \stdv{8.8}{0.11} & \stdv{5.0}{0.19} & \stdv{5.3}{0.11} & \stdv{32.9}{0.16} & \stdv{35.4}{0.30} & \stdv{16.0}{0.14} & \stdv{27.0}{0.05} \\

CoVo
& \stdv{52.6}{0.08} & \stdv{19.0}{0.16} & \stdv{1.7}{0.11} & \stdv{1.1}{0.22} & \stdv{1.2}{0.06} & \stdv{12.8}{0.24} & \stdv{17.1}{0.25} & \stdv{10.1}{0.21} & \stdv{14.5}{0.12}
& \stdv{60.0}{0.19} & \stdv{29.5}{0.11} & \stdv{7.9}{0.07} & \stdv{5.6}{0.04} & \stdv{4.8}{0.03} & \stdv{28.0}{0.20} & \stdv{31.9}{0.14} & \stdv{15.2}{0.10} & \stdv{22.9}{0.21} \\

SCOPE
& \stdv{56.2}{0.07} & \stdv{22.3}{0.10} & \stdv{4.0}{0.22} & \stdv{2.5}{0.07} & \stdv{1.5}{0.05} & \stdv{12.0}{0.18} & \stdv{13.6}{0.19} & \stdv{8.4}{0.22} & \stdv{15.1}{0.21}
& \stdv{73.8}{0.24} & \stdv{38.1}{0.05} & \stdv{9.8}{0.22} & \stdv{6.5}{0.21} & \stdv{6.1}{0.10} & \stdv{31.6}{0.31} & \stdv{30.2}{0.14} & \stdv{16.6}{0.22} & \stdv{26.6}{0.03} \\

Self-Harmony
& \stdv{60.1}{0.08} & \stdv{22.0}{0.21} & \stdv{4.0}{0.20} & \stdv{1.5}{0.12} & \stdv{1.5}{0.21} & \stdv{20.3}{0.23} & \stdv{\textbf{27.4}}{0.12} & \stdv{11.4}{0.19} & \stdv{18.5}{0.12}
& \stdv{73.6}{0.12} & \stdv{33.6}{0.18} & \stdv{8.4}{0.23} & \stdv{5.3}{0.17} & \stdv{5.3}{0.08} & \stdv{\underline{34.8}}{0.15} & \stdv{\underline{41.7}}{0.31} & \stdv{16.2}{0.25} & \stdv{27.4}{0.10} \\

Co-Reward
& \stdv{61.2}{0.23} & \stdv{23.6}{0.21} & \stdv{4.0}{0.05} & \stdv{2.6}{0.21} & \stdv{1.1}{0.15} & \stdv{18.9}{0.27} & \stdv{23.2}{0.27} & \stdv{12.3}{0.25} & \stdv{18.4}{0.06}
& \stdv{74.4}{0.06} & \stdv{38.0}{0.13} & \stdv{9.4}{0.19} & \stdv{4.8}{0.12} & \stdv{5.7}{0.15} & \stdv{33.7}{0.25} & \stdv{33.8}{0.13} & \stdv{\underline{17.2}}{0.09} & \stdv{27.1}{0.10} \\

RESTRAIN
& \stdv{58.3}{0.17} & \stdv{22.3}{0.16} & \stdv{3.6}{0.10} & \stdv{2.1}{0.15} & \stdv{1.9}{0.19} & \stdv{16.2}{0.25} & \stdv{19.8}{0.30} & \stdv{11.7}{0.16} & \stdv{17.0}{0.11}
& \stdv{71.0}{0.24} & \stdv{37.2}{0.04} & \stdv{8.2}{0.23} & \stdv{5.5}{0.17} & \stdv{4.9}{0.12} & \stdv{31.0}{0.28} & \stdv{35.2}{0.15} & \stdv{16.9}{0.23} & \stdv{26.2}{0.09} \\

SR-TTRL
& \stdv{60.9}{0.08} & \stdv{24.3}{0.23} & \stdv{4.1}{0.12} & \stdv{2.5}{0.11} & \stdv{0.7}{0.21} & \stdv{20.2}{0.25} & \stdv{24.7}{0.19} & \stdv{11.9}{0.24} & \stdv{18.7}{0.18}
& \stdv{74.3}{0.16} & \stdv{38.7}{0.13} & \stdv{9.8}{0.09} & \stdv{5.9}{0.23} & \stdv{5.5}{0.21} & \stdv{32.1}{0.19} & \stdv{31.0}{0.18} & \stdv{16.6}{0.09} & \stdv{26.7}{0.19} \\

\rowcolor{gray!20}
\textbf{GMAE$_{\text{0}}$ (Ours)}
& \stdv{63.0}{0.11} & \stdv{25.8}{0.06} & \stdv{5.5}{0.09} & \stdv{4.2}{0.04} & \stdv{2.5}{0.12} & \stdv{\textbf{22.2}}{0.18} & \stdv{25.4}{0.22} & \stdv{14.3}{0.07} & \stdv{20.4}{0.10}
& \stdv{\underline{75.7}}{0.05} & \stdv{39.7}{0.12} & \stdv{10.9}{0.08} & \stdv{\underline{8.4}}{0.06} & \stdv{6.5}{0.10} & \stdv{34.1}{0.15} & \stdv{38.8}{0.21} & \stdv{17.1}{0.09} & \stdv{28.9}{0.07} \\

\rowcolor{gray!20}
\textbf{GMAE$_{\text{CR}}$ (Ours)}
& \stdv{\textbf{63.8}}{0.06} & \stdv{\textbf{26.6}}{0.10} & \stdv{\underline{6.4}}{0.05} & \stdv{\textbf{5.0}}{0.08} & \stdv{\underline{3.0}}{0.11} & \stdv{\underline{22.0}}{0.14} & \stdv{25.0}{0.19} & \stdv{\underline{14.5}}{0.12} & \stdv{\underline{20.8}}{0.06}
& \stdv{\textbf{76.2}}{0.09} & \stdv{\underline{40.2}}{0.05} & \stdv{\underline{11.5}}{0.11} & \stdv{8.1}{0.07} & \stdv{\textbf{7.8}}{0.04} & \stdv{\textbf{35.3}}{0.17} & \stdv{\textbf{43.4}}{0.20} & \stdv{\textbf{18.3}}{0.10} & \stdv{\textbf{30.1}}{0.08} \\

\rowcolor{gray!20}
\textbf{GMAE$_{\text{SR}}$ (Ours)}
& \stdv{\underline{63.6}}{0.08} & \stdv{\underline{26.4}}{0.05} & \stdv{\textbf{6.9}}{0.12} & \stdv{\underline{4.6}}{0.09} & \stdv{\textbf{3.3}}{0.06} & \stdv{\textbf{22.2}}{0.20} & \stdv{\underline{26.3}}{0.16} & \stdv{\textbf{14.8}}{0.04} & \stdv{\textbf{21.0}}{0.11}
& \stdv{\underline{75.7}}{0.12} & \stdv{\textbf{41.2}}{0.08} & \stdv{\textbf{12.0}}{0.06} & \stdv{\textbf{9.2}}{0.10} & \stdv{\underline{7.0}}{0.05} & \stdv{34.3}{0.22} & \stdv{40.4}{0.18} & \stdv{17.5}{0.07} & \stdv{\underline{29.7}}{0.09} \\

\midrule

& \multicolumn{9}{c}{\texttt{\textbf{Qwen3-8B-Base}}}
& \multicolumn{9}{c}{\texttt{\textbf{Llama3.1-8B-Instruct}}}
\\

\midrule

Raw Model
& \stdv{68.5}{0.19} & \stdv{35.4}{0.11} & \stdv{11.3}{0.20} & \stdv{7.3}{0.07} & \stdv{6.1}{0.19} & \stdv{31.0}{0.17} & \stdv{42.8}{0.14} & \stdv{19.5}{0.06} & \stdv{27.7}{0.12}
& \stdv{47.3}{0.14} & \stdv{21.3}{0.06} & \stdv{7.2}{0.15} & \stdv{1.5}{0.13} & \stdv{1.2}{0.25} & \stdv{17.4}{0.13} & \stdv{21.5}{0.29} & \stdv{9.2}{0.14} & \stdv{15.8}{0.20} \\

{\it w/ Verifiable Reward}
& \stdv{\textit{83.1}}{0.25} & \stdv{\textit{49.4}}{0.09} & \stdv{\textit{20.4}}{0.16} & \stdv{\textit{19.1}}{0.20} & \stdv{\textit{16.3}}{0.17} & \stdv{\textit{42.9}}{0.15} & \stdv{\textit{71.2}}{0.30} & \stdv{\textit{26.0}}{0.09} & \stdv{\textit{40.9}}{0.11}
& \stdv{\textit{69.4}}{0.11} & \stdv{\textit{41.2}}{0.07} & \stdv{\textit{16.2}}{0.24} & \stdv{\textit{7.2}}{0.14} & \stdv{\textit{7.0}}{0.04} & \stdv{\textit{25.8}}{0.28} & \stdv{\textit{35.5}}{0.26} & \stdv{\textit{16.5}}{0.07} & \stdv{\textit{27.4}}{0.16} \\

\hdashline

Intuitor$^*$
& \stdv{71.1}{0.03} & \stdv{37.8}{0.04} & \stdv{11.9}{0.15} & \stdv{9.7}{0.23} & \stdv{6.8}{0.24} & \stdv{35.2}{0.14} & \stdv{60.6}{0.14} & \stdv{19.8}{0.05} & \stdv{31.6}{0.17}
& \stdv{52.1}{0.06} & \stdv{27.6}{0.11} & \stdv{7.7}{0.19} & \stdv{2.7}{0.06} & \stdv{1.6}{0.06} & \stdv{19.9}{0.17} & \stdv{27.6}{0.32} & \stdv{10.2}{0.05} & \stdv{18.7}{0.21} \\

EM-RL$^*$
& \stdv{73.4}{0.12} & \stdv{40.4}{0.25} & \stdv{13.1}{0.22} & \stdv{9.8}{0.25} & \stdv{7.9}{0.10} & \stdv{35.8}{0.29} & \stdv{65.3}{0.22} & \stdv{20.5}{0.09} & \stdv{33.3}{0.22}
& \stdv{58.9}{0.15} & \stdv{34.2}{0.18} & \stdv{9.0}{0.07} & \stdv{3.5}{0.17} & \stdv{2.2}{0.04} & \stdv{20.3}{0.24} & \stdv{26.7}{0.27} & \stdv{13.1}{0.07} & \stdv{21.0}{0.17} \\

TTRL
& \stdv{77.8}{0.25} & \stdv{43.6}{0.03} & \stdv{13.1}{0.14} & \stdv{10.5}{0.09} & \stdv{6.1}{0.16} & \stdv{39.7}{0.24} & \stdv{64.1}{0.14} & \stdv{20.9}{0.17} & \stdv{34.5}{0.04}
& \stdv{63.8}{0.24} & \stdv{36.6}{0.10} & \stdv{11.6}{0.16} & \stdv{5.0}{0.19} & \stdv{4.1}{0.21} & \stdv{21.1}{0.30} & \stdv{31.8}{0.13} & \stdv{14.3}{0.12} & \stdv{23.5}{0.17} \\

CoVo
& \stdv{69.9}{0.11} & \stdv{37.1}{0.13} & \stdv{11.9}{0.09} & \stdv{9.1}{0.15} & \stdv{7.0}{0.06} & \stdv{32.3}{0.13} & \stdv{52.8}{0.16} & \stdv{20.4}{0.04} & \stdv{30.1}{0.05}
& \stdv{51.3}{0.15} & \stdv{29.5}{0.25} & \stdv{10.2}{0.14} & \stdv{2.4}{0.17} & \stdv{1.9}{0.05} & \stdv{18.5}{0.22} & \stdv{29.3}{0.18} & \stdv{12.4}{0.16} & \stdv{19.4}{0.13} \\

SCOPE
& \stdv{77.5}{0.23} & \stdv{43.3}{0.07} & \stdv{12.5}{0.20} & \stdv{11.5}{0.16} & \stdv{7.1}{0.22} & \stdv{38.6}{0.30} & \stdv{59.3}{0.24} & \stdv{19.5}{0.10} & \stdv{33.7}{0.15}
& \stdv{60.3}{0.15} & \stdv{32.8}{0.13} & \stdv{12.1}{0.04} & \stdv{4.2}{0.08} & \stdv{3.0}{0.09} & \stdv{20.0}{0.22} & \stdv{30.3}{0.21} & \stdv{11.9}{0.22} & \stdv{21.8}{0.17} \\

Self-Harmony
& \stdv{76.3}{0.08} & \stdv{41.1}{0.22} & \stdv{12.7}{0.15} & \stdv{11.5}{0.18} & \stdv{6.5}{0.10} & \stdv{40.7}{0.30} & \stdv{\textbf{70.5}}{0.21} & \stdv{21.9}{0.16} & \stdv{35.2}{0.11}
& \stdv{61.7}{0.15} & \stdv{33.6}{0.09} & \stdv{8.0}{0.16} & \stdv{3.9}{0.19} & \stdv{1.8}{0.06} & \stdv{21.8}{0.30} & \stdv{28.4}{0.14} & \stdv{13.0}{0.24} & \stdv{21.5}{0.08} \\

Co-Reward
& \stdv{78.0}{0.16} & \stdv{43.5}{0.15} & \stdv{13.4}{0.06} & \stdv{10.4}{0.05} & \stdv{6.9}{0.25} & \stdv{40.6}{0.23} & \stdv{65.3}{0.14} & \stdv{21.1}{0.15} & \stdv{34.9}{0.17}
& \stdv{65.1}{0.12} & \stdv{36.8}{0.09} & \stdv{12.1}{0.19} & \stdv{5.3}{0.20} & \stdv{4.7}{0.09} & \stdv{21.6}{0.27} & \stdv{32.4}{0.24} & \stdv{15.3}{0.19} & \stdv{24.2}{0.23} \\

RESTRAIN
& \stdv{77.3}{0.19} & \stdv{44.8}{0.10} & \stdv{13.8}{0.17} & \stdv{9.8}{0.10} & \stdv{8.1}{0.12} & \stdv{39.0}{0.12} & \stdv{65.0}{0.21} & \stdv{20.7}{0.23} & \stdv{34.8}{0.20}
& \stdv{64.2}{0.11} & \stdv{30.6}{0.21} & \stdv{10.9}{0.23} & \stdv{4.3}{0.14} & \stdv{4.0}{0.18} & \stdv{20.7}{0.20} & \stdv{30.8}{0.31} & \stdv{13.9}{0.11} & \stdv{22.4}{0.16} \\

SR-TTRL
& \stdv{77.3}{0.16} & \stdv{44.5}{0.18} & \stdv{13.8}{0.18} & \stdv{11.1}{0.11} & \stdv{7.2}{0.17} & \stdv{39.2}{0.31} & \stdv{63.7}{0.31} & \stdv{20.1}{0.14} & \stdv{34.6}{0.06}
& \stdv{65.0}{0.06} & \stdv{36.3}{0.13} & \stdv{12.1}{0.15} & \stdv{5.3}{0.25} & \stdv{5.5}{0.15} & \stdv{22.0}{0.30} & \stdv{31.6}{0.29} & \stdv{13.5}{0.06} & \stdv{23.9}{0.12} \\

\rowcolor{gray!20}
\textbf{GMAE$_{\text{0}}$ (Ours)}
& \stdv{79.8}{0.05}
& \stdv{46.2}{0.11}
& \stdv{16.2}{0.07}
& \stdv{12.5}{0.13}
& \stdv{\underline{9.2}}{0.06}
& \stdv{\textbf{41.7}}{0.19}
& \stdv{67.4}{0.23}
& \stdv{\underline{23.5}}{0.08}
& \stdv{\underline{37.1}}{0.10}
& \stdv{66.5}{0.12}
& \stdv{38.4}{0.06}
& \stdv{13.2}{0.09}
& \stdv{\underline{6.5}}{0.04}
& \stdv{5.8}{0.11}
& \stdv{22.9}{0.16}
& \stdv{32.8}{0.21}
& \stdv{15.4}{0.05}
& \stdv{25.2}{0.09} \\

\rowcolor{gray!20}
\textbf{GMAE$_{\text{CR}}$ (Ours)}
& \stdv{\textbf{81.6}}{0.09}
& \stdv{\textbf{47.0}}{0.04}
& \stdv{\textbf{17.4}}{0.10}
& \stdv{\textbf{13.5}}{0.06}
& \stdv{\textbf{9.5}}{0.12}
& \stdv{\underline{41.3}}{0.15}
& \stdv{\underline{69.2}}{0.18}
& \stdv{\textbf{24.3}}{0.11}
& \stdv{\textbf{38.0}}{0.07}
& \stdv{\underline{67.7}}{0.06}
& \stdv{\textbf{39.8}}{0.10}
& \stdv{\textbf{14.0}}{0.05}
& \stdv{\textbf{6.6}}{0.08}
& \stdv{\textbf{6.4}}{0.13}
& \stdv{\textbf{23.9}}{0.17}
& \stdv{\textbf{34.2}}{0.20}
& \stdv{\textbf{16.7}}{0.09}
& \stdv{\textbf{26.2}}{0.06}
\\

\rowcolor{gray!20}
\textbf{GMAE$_{\text{SR}}$ (Ours)}
& \stdv{\underline{80.4}}{0.07}
& \stdv{\underline{46.4}}{0.12}
& \stdv{\underline{16.5}}{0.04}
& \stdv{\underline{13.1}}{0.09}
& \stdv{8.6}{0.05}
& \stdv{40.7}{0.21}
& \stdv{66.6}{0.17}
& \stdv{22.6}{0.06}
& \stdv{36.9}{0.11}
& \stdv{\textbf{68.0}}{0.10}
& \stdv{\underline{39.2}}{0.05}
& \stdv{\underline{13.6}}{0.12}
& \stdv{6.1}{0.07}
& \stdv{\underline{6.0}}{0.09}
& \stdv{\underline{23.3}}{0.19}
& \stdv{\underline{33.7}}{0.16}
& \stdv{\underline{16.2}}{0.11}
& \stdv{\underline{25.8}}{0.08}
\\

\bottomrule

\end{tabular}
}

\label{tab:main}
\vspace{-0.2in}

\end{table*}

\vspace{-0.05in}
\subsection{Setup}

\paragraph{Models.}
We conduct our main experiments on \texttt{Qwen3-1.7B-Base}, \texttt{Qwen3-4B-Base}, and \texttt{Qwen3-8B-Base}~\colorcitep{yang2025qwen3} to evaluate GMAE across model scales.
We also include \texttt{Llama3.1-8B-Instruct}~\colorcitep{grattafiori2024llama} to assess its adaptability across model families.

\vspace{-0.1in}
\paragraph{Datasets.}
For training, we use the math dataset \texttt{DAPO-14K}~\colorcitep{yu2026dapo}.
For evaluation, we select eight benchmarks across mathematics, science, knowledge, and coding domains.
For mathematics, we use \texttt{MATH500}~\colorcitep{hendrycks2021measuring}, \texttt{AMC}, \texttt{AIME24}, \texttt{AIME25}, and \texttt{AIME26}.
The out-of-domain (OOD) benchmarks include \texttt{GPQA}~\colorcitep{rein2023gpqa} for science, \texttt{MMLU-Pro}~\colorcitep{wang2024mmlu} for knowledge, and \texttt{LiveCodeBench-v6}~\colorcitep{jain2024livecodebench} for coding.

\vspace{-0.1in}
\paragraph{Baselines.}
We primarily compare GMAE with {\it ensemble-based} methods, including TTRL~\colorcitep{zuo2026ttrl}, CoVo~\colorcitep{zhang2026consistent}, SCOPE~\colorcitep{wang2026beyond}, Self-Harmony~\colorcitep{wang2026selfharmony}, RESTRAIN~\colorcitep{yu2026restrain}, Co-Reward~\colorcitep{zhang2026corewarding}, and SR-TTRL~\colorcitep{wu2026srttrl}.
Detailed descriptions are provided in \textcolor{deepred}{Appendix~\ref{app:baseline-methods}}.
We also include two {\it probability-based} methods: Intuitor~\colorcitep{zhao2025learning}, which uses self-certainty as its reward, and EM-RL~\colorcitep{agarwal2026unreasonable}, which uses trajectory entropy as its reward.
We further include an RLVR reference to measure the remaining gap to verifiable supervision.
We faithfully reproduce all baselines using official implementations when available; otherwise, we implement them based on the original papers.

\vspace{-0.1in}
\paragraph{Implementation.}
We adopt GRPO~\colorcitep{shao2024deepseekmath} as the RL backbone and train each base model for one epoch with batch size $64$.
We use a learning rate of $1\times10^{-6}$ with $5\%$ warmup followed by cosine decay~\colorcitep{loshchilov2017sgdr}.
For each prompt, we sample $G=16$ main and $G'=16$ auxiliary responses with temperature $T=1.0$ and no top-$p$ or top-$k$ truncation, forming a candidate pool of $32$ responses.
All $32$ responses are generated fully in parallel during rollout.
We implement all experiments with \texttt{verl} on NVIDIA A100 80GB GPUs.

\vspace{-0.1in}
\paragraph{Evaluation.}
For each evaluation dataset, we report \textsc{Mean@16}, defined as the average \textsc{Pass@1} accuracy over 16 sampled responses per problem. The subscript reports the standard deviation of per-problem \textsc{Pass@1} outcomes aggregated over the evaluation set.
We use temperature $T=0.6$, top-$p=0.95$, and top-$k=20$ for sampling.
During training, we evaluate checkpoints on \texttt{DAPO-500}~\colorcitep{yu2026dapo} every $20$ steps and at training end, selecting the checkpoint with the highest validation \textsc{Mean@16} for final benchmarking.

\vspace{-0.1in}
\paragraph{Fair Cost Control.}
Training cost is reflected in two aspects: {\bf GPU memory} and {\bf training duration}.
For GPU memory, all methods use $G=16$ responses for policy updates.
For a fair comparison, the baselines also generate $32$ responses per prompt, with the additional responses used only for reward-reference construction. 
This enlarged rollout scheme has been shown to improve performance~\colorcitep{zuo2026ttrl}.
For training duration, it varies with the computation required by each reward construction and is discussed separately in \textcolor{deepred}{Section~\ref{sec:training-behavior}}.

\vspace{-0.05in}
\subsection{Main Results}

\paragraph{\textcolor{deepred}{Overall Comparisons:} Robust Improvement and Generalization.}
In \textcolor{deepred}{Table \ref{tab:main}}, GMAE outperforms existing self-rewarding baselines across 30 of 32 model-benchmark combinations.
It achieves improvements of 2.0-2.8 points over the strongest baselines, corresponding to 8.0\%-12.3\% relative gains across different models.
Notably, GMAE approaches the performance of RLVR with verifiable rewards. For example, the performance gap on \texttt{Qwen3-1.7B-Base} is reduced to only 1.0 point.
Moreover, although trained exclusively on mathematics, GMAE maintains strong performance on OOD benchmarks, demonstrating its broad generalization ability.

\vspace{-0.1in}
\paragraph{\textcolor{deepred}{Controlled Comparisons:} GMAE as a Strong General Strategy.}
To isolate the contribution of GMAE from the underlying self-rewarding rules, we compare each GMAE instantiation with its base version: GMAE$_0$ with TTRL, GMAE$_{\mathrm{CR}}$ with Co-Reward, and GMAE$_{\mathrm{SR}}$ with SR-TTRL.
Across all comparisons, GMAE consistently improves the corresponding methods, yielding average relative gains of 8.2\%, 10.3\%, and 9.5\%, respectively.
Meanwhile, the relative ranking among different self-rewarding rules is largely preserved after applying GMAE, indicating that GMAE improves reward utilization while retaining their individual strengths.
Notably, even the weakest GMAE instantiation consistently outperforms the strongest non-GMAE baseline across all evaluated models.
These results demonstrate that the improvement comes from the GMAE strategy for leveraging self-reward signals, rather than from any specific self-rewarding rule alone.

\vspace{-0.05in}
\subsection{Detailed Training Behavior}
\label{sec:training-behavior}

\paragraph{\textcolor{deepred}{Training Dynamic:} Stability Analysis.}
Stable training dynamics are essential for sustained self-improvement.
\textcolor{deepred}{Figure~\ref{fig:dynamic}} tracks performance on the \texttt{DAPO-500} validation set throughout training.
Several baselines improve initially but then degrade sharply, exhibiting a ``rise-then-collapse'' pattern.
Notably, both probability-based methods show this behavior, suggesting that probability-based supervision can be unreliable~\colorcitep{zhang2025no}.
In contrast, GMAE improves steadily and converges without collapse, demonstrating substantially more reliable training dynamics.

\begin{figure}[H]
\vspace{-0.15in}
    \centering
    \includegraphics[width=\linewidth]{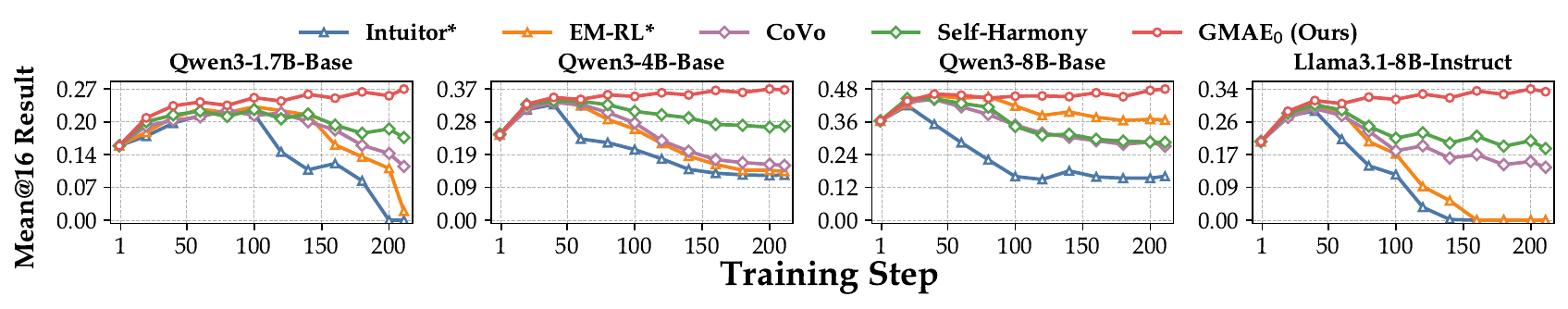}
    \vspace{-0.3in}
    \caption{
    \textbf{Training Dynamics} on the \texttt{DAPO-500} validation set over training steps.
    }
    \label{fig:dynamic}
\vspace{-0.2in}
\end{figure}

\begin{wrapfigure}{r}{0.45\textwidth}
    \centering
    \vspace{-0.17in}
    \includegraphics[width=0.45\textwidth]{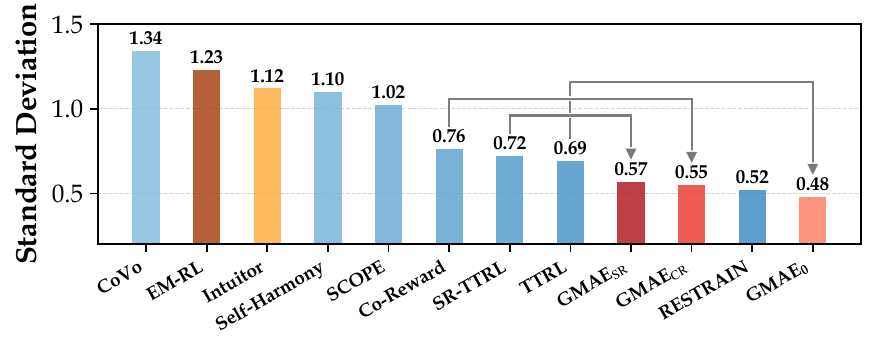}
    \vspace{-0.3in}
    \caption{
    \textbf{Multi-Seed Training Variance.}
    Results show the average \textsc{Mean@16} standard deviation across eight benchmarks for \texttt{Qwen3-8B-Base} over five training seeds.
    }
    \label{fig:variance}
    \vspace{-0.2in}
\end{wrapfigure}

\vspace{-0.1in}
\paragraph{\textcolor{deepred}{Training Robustness:} Multi-Seed Training Variance.}
As analyzed in \textcolor{deepred}{Section \ref{sec:existing-limitations}}, randomness in rollout sampling may flip the sign of the estimated advantage for the same response, substantially affecting policy updates. We therefore test whether GMAE produces more consistent outcomes across runs. We evaluate each method over five runs on \texttt{Qwen3-8B-Base}, keeping the hardware, initialization, training data, batch order, and hyperparameters fixed while varying only the rollout-sampling seed. \textcolor{deepred}{Figure~\ref{fig:variance}} reports the standard deviation of \textsc{Mean@16} across runs, averaged over eight benchmarks. GMAE$_0$ achieves the lowest variance at 0.48, with lower variance in every controlled comparison: TTRL $\rightarrow$ GMAE$_0$ (0.69 $\rightarrow$ 0.48), Co-Reward $\rightarrow$ GMAE$_{\text{CR}}$ (0.76 $\rightarrow$ 0.55), and SR-TTRL $\rightarrow$ GMAE$_{\text{SR}}$ (0.72 $\rightarrow$ 0.57). These results show that beyond improving absolute performance, GMAE also makes training more robust.

\begin{wrapfigure}{r}{0.43\textwidth}
    \centering
    \vspace{-0.2in}
    \includegraphics[
        width=0.43\textwidth
    ]{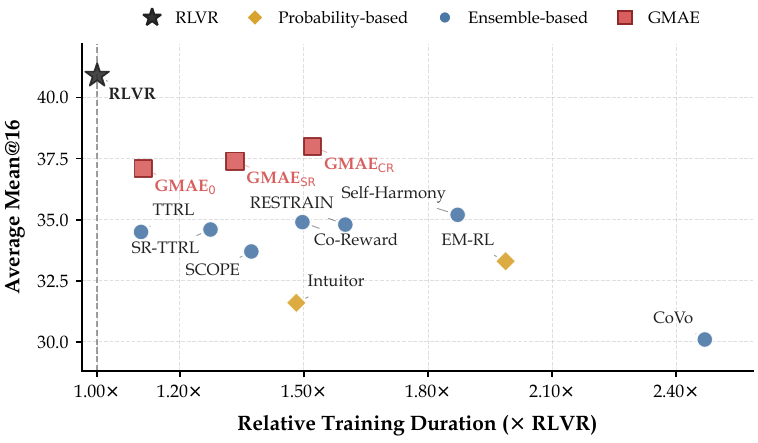}
    \vspace{-0.3in}
    \caption{
    \textbf{Training Duration vs. Performance.}
    Each point represents one method. Duration is normalized to RLVR, and performance is the average \textsc{Mean@16} across eight benchmarks on \texttt{Qwen3-8B-Base}.
    }
    \label{fig:training-cost}
    \vspace{-0.2in}
\end{wrapfigure}

\vspace{-0.1in}
\paragraph{\textcolor{deepred}{Training Duration:} Cost Analysis.}
With equal rollout sampling budgets, we compare end-to-end training time under the same hardware and settings, normalized to RLVR.
\textcolor{deepred}{Figure~\ref{fig:training-cost}} shows that all three GMAE instantiations remain within the low-to-moderate cost range of existing baselines and add little overhead over their corresponding base versions.
This limited overhead arises because {\it marginalization reuses cached responses and requires only lightweight scalar computation}.
For example, GMAE$_0$ introduces only $G(K-1)$ lightweight voting operations per prompt and at most $2^G$ scalar normalizations over TTRL, resulting in nearly identical training time.
Meanwhile, the more elaborate GMAE instantiations require additional computation from their underlying self-rewarding rules but achieve correspondingly stronger performance, yielding a favorable cost-performance trade-off.

\vspace{-0.05in}
\section{Extended Analysis}

\vspace{-0.05in}

\subsection{Why Does GMAE Work? A Correct-Response
Frequency Argument}
\label{sec:why-gmae-works}

The key motivation for GMAE is that a correct response may receive a positive reward only in certain group contexts, which a single realization may miss.
By retaining reward realizations across contexts, GMAE preserves even sparse positive evidence.
Extending this intuition to model performance, we initially expected GMAE to prevent rare correct responses from disappearing while keeping common ones frequent, as they are consistently rewarded across contexts.
We test this explanation by tracking correct-response frequencies throughout training.

\begin{wrapfigure}{r}{0.45\textwidth}
    \centering
    \vspace{-0.3in}
    \includegraphics[
        width=0.45\textwidth
    ]{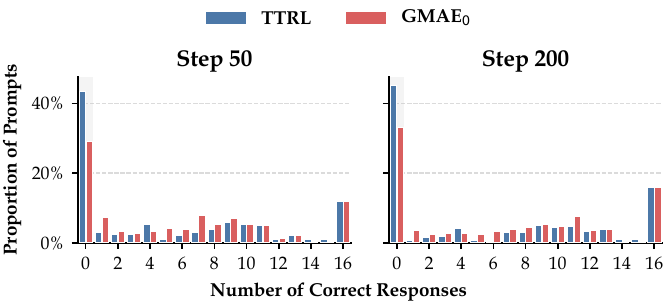}
    \vspace{-0.25in}
    \caption{
    \textbf{Distribution of Correct Responses per Prompt.}
    Illustrative proportions of training prompts by the number of correct responses among $G=16$ sampled responses under GMAE$_0$ and TTRL at steps 50 and 200.
    }
    \label{fig:correct-response-retention}
    \vspace{-0.2in}
\end{wrapfigure}

We compare illustrative count distributions for GMAE$_0$ and its non-marginalized version TTRL on \texttt{Qwen3-8B-Base} at steps 50 and 200, representing early and near-final stages of the 211-step training.
As shown in \textcolor{deepred}{Figure~\ref{fig:correct-response-retention}}, TTRL has more prompts with no correct responses, whereas GMAE$_0$ has more prompts with a small or moderate number of correct responses.
The two distributions have almost the same mass at 16 correct responses.
This pattern suggests that GMAE helps keep rare correct responses from disappearing while maintaining success on prompts that are already solved consistently.

\vspace{-0.05in}
\subsection{Ablation Study}

\begin{wrapfigure}{r}{0.45\textwidth}
    \centering
    \vspace{-0.4in}
    \includegraphics[width=0.45\textwidth]{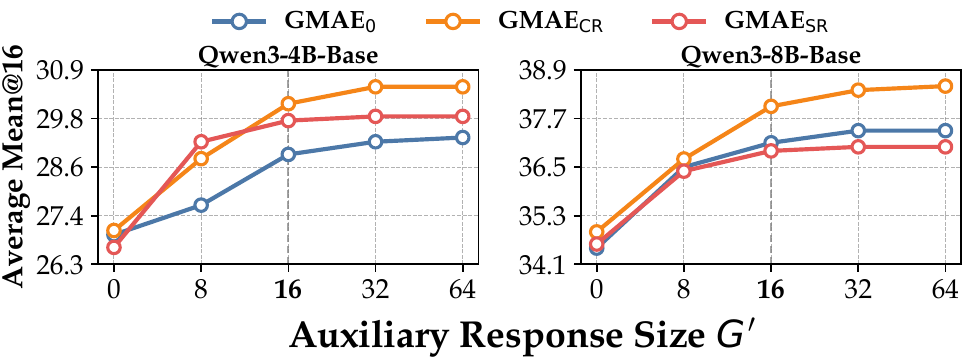}
    \vspace{-0.28in}
    \caption{
    \textbf{Auxiliary Response Size Ablation.}
    Average \textsc{Mean@16} across all eight benchmarks with different auxiliary response sizes.
    }
    \label{fig:auxiliary-size}
    \vspace{-0.2in}
\end{wrapfigure}

\paragraph{Auxiliary Response Size $G'$.}
\textcolor{deepred}{Eq.\ref{eq:bound}} establishes that the upper bound on the estimation error of $\hat{P}_{r_i}$ decreases with larger $G'$.
We validate this result empirically by varying $G'$ under otherwise identical configurations and performing multiple runs. When $G'=0$, the GMAE instantiation reduces to its base version.
As shown in \textcolor{deepred}{Figure~\ref{fig:auxiliary-size}}, increasing $G'$ consistently improves performance, but the marginal benefit drops markedly beyond $G'=16$.
This suggests that our default choice of $G'=16$ provides a strong balance between performance and rollout cost.

\vspace{-0.1in}
\paragraph{Advantage Scale Calibration Mechanism.}
\textcolor{deepred}{Section~\ref{sec:gmae}} theoretically motivates advantage scale calibration and explains why calibration should be applied at the batch rather than prompt level.
Now we empirically validate this design. \textcolor{deepred}{Table~\ref{tab:calibration}}  shows that batch-level calibration performs best on both models, whereas prompt-level calibration performs worst, even underperforming no calibration.

\begin{table}[H]
\vspace{-0.1in}
\caption{
\textbf{Calibration Ablation on GMAE$_0$.}
\textsc{Mean@16} results for \texttt{Qwen3-8B-Base}.
}
\label{tab:calibration-8b}
\vspace{-0.1in}
\centering
\footnotesize
\renewcommand{\arraystretch}{1.1}
\setlength{\tabcolsep}{1.5mm}

\resizebox{\linewidth}{!}{
\begin{tabular}{lccccccccc}
\toprule

& \textbf{MATH500}
& \textbf{AMC}
& \textbf{AIME24}
& \textbf{AIME25}
& \textbf{AIME26}
& \textbf{GPQA}
& \textbf{MMLU-Pro}
& \textbf{LiveCode}
& \textbf{\textit{Average}}
\\

\midrule

w/o Calibration
& 78.4 & 45.2 & 16.0 & 10.3 & \textbf{9.7} & 39.2 & 63.6
& 22.4 & 35.6 \\

w/ Prompt-level Calibration
& 78.0 & 44.7 & 15.5 & 10.8 & 8.6 & 39.4 & 63.8 & 21.3 & 35.3 \\

\rowcolor{gray!20}
{\bf w/ Batch-level Calibration (Ours)}
& \textbf{79.8} & \textbf{46.2} & \textbf{16.2} & \textbf{12.5}
& 9.2 & \textbf{41.7} & \textbf{67.4} & \textbf{23.5}
& \textbf{37.1} \\

\bottomrule
\end{tabular}
}
\vspace{-0.15in}
\label{tab:calibration}
\end{table}

\begin{wrapfigure}{r}{0.45\textwidth}
    \centering
    \vspace{-0.27in}
    \includegraphics[width=0.45\textwidth]{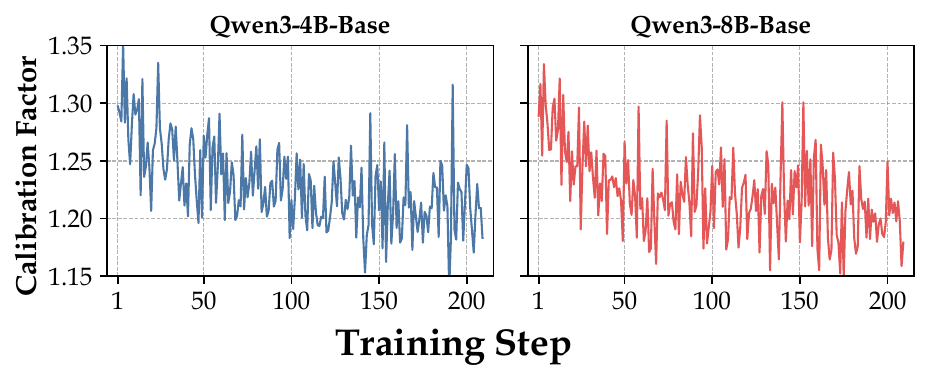}
    \vspace{-0.3in}
    \caption{
    \textbf{Calibration Factor $\lambda$ Dynamics.}
    }
    \label{fig:calibration}
    \vspace{-0.1in}
\end{wrapfigure}

\textcolor{deepred}{Figure~\ref{fig:calibration}} further shows how the calibration factor $\lambda$ (\textcolor{deepred}{Eq.\ref{eq:gmae-calibration}}) evolves during training.
Its value remains between 1.15 and 1.35, confirming that GMAE consistently requires moderate scale compensation.
The factor gradually decreases as training proceeds, indicating that the gap between the non-marginalized and marginalized advantage scales narrows over time.

\vspace{-0.05in}
\subsection{Scalability Study}

\begin{wraptable}{r}{0.45\textwidth}
\centering
\vspace{-0.17in}
\caption{
\textbf{Scalability Across Training Settings.}
Average \textsc{Mean@16} across all eight benchmarks for
\texttt{Qwen3-8B-Base}.
Full benchmark-wise results are in
\textcolor{deepred}{Appendix~\ref{app:scalability-results}}.
}
\label{tab:scalability}
\vspace{-0.08in}

\footnotesize
\renewcommand{\arraystretch}{1.1}
\setlength{\tabcolsep}{1.5mm}

\resizebox{\linewidth}{!}{
\begin{tabular}{lcccc}
\toprule

\multirow{2}{*}{\textbf{Method}}
& \multicolumn{2}{c}{\textbf{RL Backbones}}
& \multicolumn{2}{c}{\textbf{Training Datasets}}
\\

\cmidrule(lr){2-3}
\cmidrule(lr){4-5}

& \textbf{GSPO}
& \textbf{REINFORCE++}
& \textbf{\texttt{Open-RS}}
& \textbf{\texttt{MATH-8K}}
\\

\midrule

TTRL
& 35.8 & 32.8 & 34.0 & 33.2 \\

Self-Harmony
& 36.1 & 33.6 & 34.9 & 33.0 \\

Co-Reward
& 37.3 & 33.9 & 35.3 & 34.3 \\

SR-TTRL
& 36.8 & 33.6 & 35.2 & 34.3 \\

\rowcolor{gray!20}
\textbf{GMAE$_0$ (Ours)}
& 38.1 & 36.4 & 36.2 & 35.5 \\

\rowcolor{gray!20}
\textbf{GMAE$_{\text{CR}}$ (Ours)}
& \textbf{39.3}
& \textbf{37.5}
& \textbf{37.3}
& 36.2 \\

\rowcolor{gray!20}
\textbf{GMAE$_{\text{SR}}$ (Ours)}
& 38.7
& 37.1
& 36.6
& \textbf{36.5} \\

\bottomrule
\end{tabular}
}

\vspace{-0.2in}
\end{wraptable}

We evaluate GMAE's scalability along two dimensions: {\bf RL backbones} and {\bf training datasets}.
For the former, we replace the default GRPO backbone with GSPO~\colorcitep{zheng2025group} and REINFORCE++~\colorcitep{hu2025reinforce++}, keeping all other settings fixed.
For the latter, we train all methods on \texttt{Open-RS}~\colorcitep{dang2025reinforcement} and \texttt{MATH-8K}~\colorcitep{hendrycks2021measuring}, in addition to \texttt{DAPO-14K}.
As shown in \textcolor{deepred}{Table~\ref{tab:scalability}}, GMAE consistently retains its advantages over representative strong baselines from \textcolor{deepred}{Table~\ref{tab:main}}, demonstrating scalability across the two training settings.

\vspace{-0.1in}
\section{Related Work}

\paragraph{Self-Evolving LLMs.}
Self-evolving LLMs improve through self-generated experience~\colorcitep{tao2024selfevolution,gao2025selfevolving}.
They evolve data through self-generation or instruction evolution~\colorcitep{wang2023selfinstruct,xu2024wizardlm} and failure-targeted augmentation~\colorcitep{lee2024llm2llm}; reasoning through iterative refinement~\colorcitep{madaan2023selfrefine,shinn2023reflexion}, latent-rationale bootstrapping~\colorcitep{zelikman2022star,zelikman2024quietstar}, and structure discovery~\colorcitep{zhou2024selfdiscover}; and policies or tasks through reinforced self-training, self-play~\colorcitep{gulcehre2023rest,chen2024spin}, and proposer--solver co-evolution~\colorcitep{zhao2025absolutezero,huang2025rzero}.
Yet many require demonstrations, labels, environmental feedback, or verifiers.
We instead study \textit{zero-label} evolution from prompts alone, without ground-truth answers or external reward models.

\vspace{-0.1in}
\paragraph{Self-Rewarding RL.}
Self-rewarding RL uses self-generated supervision, extending RLAIF~\colorcitep{bai2022constitutional,lee2024rlaif} to closed self-evaluation loops~\colorcitep{yuan2024selfrewarding,wang2024selftaught,wu2024metarewarding}.
For reasoning, probability-based methods score individual rollouts using confidence or self-certainty~\colorcitep{zhao2025learning,li2025confidence}, while entropy-based objectives reward concentrated token distributions~\colorcitep{agarwal2026unreasonable}.
Ensemble-based methods derive pseudo-supervision from multiple rollouts, beginning with answer-frequency consensus~\colorcitep{zuo2026ttrl} and extending to trajectory consistency~\colorcitep{zhang2026consistent}, weighted or subgroup voting~\colorcitep{wang2026beyond}, transformed-view agreement~\colorcitep{wang2026selfharmony}, complementary or teacher-guided consensus~\colorcitep{zhang2026corewarding}, and self-reflection~\colorcitep{wu2026srttrl}.
In contrast, GMAE supports different self-reward rules and addresses their shared dependence on sampled rollout groups.

\vspace{-0.05in}
\section{Conclusion}
\vspace{-0.05in}

In this paper, we introduced GMAE, a general advantage-estimation strategy for self-rewarding RL. By marginalizing rewards over possible group contexts, GMAE provides more reliable optimization signals while remaining compatible with different self-rewarding rules. Experiments across multiple models and benchmarks demonstrate consistent improvements and stable training, supporting GMAE as an effective approach to zero-label self-evolution.

\section*{AI Use Statement}
In this work, we used generative AI tools for writing polishing and the retrieval of some related work. We have verified the correctness of the AI-generated content and take responsibility for the final content of this work, including any text, claims, or artifacts produced with the aid of generative AI.

\bibliography{iclr2027_conference}

@inproceedings{wang2023selfinstruct,
  title={Self-instruct: Aligning language models with self-generated instructions},
  author={Wang, Yizhong and Kordi, Yeganeh and Mishra, Swaroop and Liu, Alisa and Smith, Noah A and Khashabi, Daniel and Hajishirzi, Hannaneh},
  booktitle={Proceedings of the 61st annual meeting of the association for computational linguistics (volume 1: long papers)},
  pages={13484--13508},
  year={2023}
}

@article{zelikman2022star,
  title={Star: Bootstrapping reasoning with reasoning},
  author={Zelikman, Eric and Wu, Yuhuai and Mu, Jesse and Goodman, Noah},
  journal={Advances in Neural Information Processing Systems},
  volume={35},
  pages={15476--15488},
  year={2022}
}

@article{chen2024spin,
  title={Self-play fine-tuning converts weak language models to strong language models},
  author={Chen, Zixiang and Deng, Yihe and Yuan, Huizhuo and Ji, Kaixuan and Gu, Quanquan},
  journal={arXiv preprint arXiv:2401.01335},
  year={2024}
}

@article{madaan2023selfrefine,
  title={Self-refine: Iterative refinement with self-feedback},
  author={Madaan, Aman and Tandon, Niket and Gupta, Prakhar and Hallinan, Skyler and Gao, Luyu and Wiegreffe, Sarah and Alon, Uri and Dziri, Nouha and Prabhumoye, Shrimai and Yang, Yiming and others},
  journal={Advances in neural information processing systems},
  volume={36},
  pages={46534--46594},
  year={2023}
}

@article{shinn2023reflexion,
  title={Reflexion: Language agents with verbal reinforcement learning},
  author={Shinn, Noah and Cassano, Federico and Gopinath, Ashwin and Narasimhan, Karthik and Yao, Shunyu},
  journal={Advances in neural information processing systems},
  volume={36},
  pages={8634--8652},
  year={2023}
}

@inproceedings{xu2024wizardlm,
  title={WizardLM: Empowering large pre-trained language models to follow complex instructions},
  author={Xu, Can and Sun, Qingfeng and Zheng, Kai and Geng, Xiubo and Zhao, Pu and Feng, Jiazhan and Tao, Chongyang and Lin, Qingwei and Jiang, Daxin},
  booktitle={International Conference on Learning Representations},
  volume={2024},
  pages={30745--30766},
  year={2024}
}

@inproceedings{lee2024llm2llm,
  title={Llm2llm: Boosting llms with novel iterative data enhancement},
  author={Lee, Nicholas and Wattanawong, Thanakul and Kim, Sehoon and Mangalam, Karttikeya and Shen, Sheng and Anumanchipalli, Gopala and Mahoney, Michael and Keutzer, Kurt and Gholami, Amir},
  booktitle={Findings of the Association for Computational Linguistics: ACL 2024},
  pages={6498--6526},
  year={2024}
}

@inproceedings{
    zelikman2024quietstar,
    title={Quiet-{ST}aR: Language Models Can Teach Themselves to Think Before Speaking},
    author={Eric Zelikman and Georges Raif Harik and Yijia Shao and Varuna Jayasiri and Nick Haber and Noah Goodman},
    booktitle={First Conference on Language Modeling},
    year={2024}
}

@article{zhou2024selfdiscover,
  title={Self-discover: Large language models self-compose reasoning structures},
  author={Zhou, Pei and Pujara, Jay and Ren, Xiang and Chen, Xinyun and Cheng, Heng-Tze and Le, Quoc V and Zhou, Denny and Mishra, Swaroop and Zheng, Huaixiu S and others},
  journal={Advances in Neural Information Processing Systems},
  volume={37},
  pages={126032--126058},
  year={2024}
}

@article{gulcehre2023rest,
  title={Reinforced self-training (rest) for language modeling},
  author={Gulcehre, Caglar and Paine, Tom Le and Srinivasan, Srivatsan and Konyushkova, Ksenia and Weerts, Lotte and Sharma, Abhishek and Siddhant, Aditya and Ahern, Alex and Wang, Miaosen and Gu, Chenjie and others},
  journal={arXiv preprint arXiv:2308.08998},
  year={2023}
}

@article{tao2024selfevolution,
  title={A survey on self-evolution of large language models},
  author={Tao, Zhengwei and Lin, Ting-En and Chen, Xiancai and Li, Hangyu and Wu, Yuchuan and Li, Yongbin and Jin, Zhi and Huang, Fei and Tao, Dacheng and Zhou, Jingren},
  journal={arXiv preprint arXiv:2404.14387},
  year={2024}
}

@article{gao2025selfevolving,
  title={A survey of self-evolving agents: What, when, how, and where to evolve on the path to artificial super intelligence},
  author={Gao, Huan-ang and Geng, Jiayi and Hua, Wenyue and Hu, Mengkang and Juan, Xinzhe and Liu, Hongzhang and Liu, Shilong and Qiu, Jiahao and Qi, Xuan and Wu, Yiran and others},
  journal={arXiv preprint arXiv:2507.21046},
  year={2025}
}

@article{zhao2025absolutezero,
  title={Absolute zero: Reinforced self-play reasoning with zero data},
  author={Zhao, Andrew and Wu, Yiran and Wu, Tong and Xu, Quentin and Yue, Yang and Lin, Matthieu and Wang, Shenzhi and Wu, Qingyun and Zheng, Zilong and Huang, Gao},
  journal={Advances in Neural Information Processing Systems},
  volume={38},
  pages={105816--105879},
  year={2026}
}

@inproceedings{huang2025rzero,
  title={R-zero: Self-evolving reasoning llm from zero data},
  author={Huang, Chengsong and Yu, Wenhao and Wang, Xiaoyang and Zhang, Hongming and Li, Zongxia and Li, Ruosen and Huang, Jiaxin and Mi, Haitao and Yu, Dong},
  booktitle={International Conference on Learning Representations},
  volume={2026},
  pages={130770--130790},
  year={2026}
}

@article{schulman2017proximal,
  title={Proximal policy optimization algorithms},
  author={Schulman, John and Wolski, Filip and Dhariwal, Prafulla and Radford, Alec and Klimov, Oleg},
  journal={arXiv preprint arXiv:1707.06347},
  year={2017}
}

@article{ouyang2022training,
  title={Training language models to follow instructions with human feedback},
  author={Ouyang, Long and Wu, Jeffrey and Jiang, Xu and Almeida, Diogo and Wainwright, Carroll and Mishkin, Pamela and Zhang, Chong and Agarwal, Sandhini and Slama, Katarina and Ray, Alex and others},
  journal={Advances in neural information processing systems},
  volume={35},
  pages={27730--27744},
  year={2022}
}

@inproceedings{loshchilov2017sgdr,
    title={{SGDR}: Stochastic Gradient Descent with Warm Restarts},
    author={Ilya Loshchilov and Frank Hutter},
    booktitle={International Conference on Learning Representations},
    year={2017},
}

@article{shao2024deepseekmath,
  title={Deepseekmath: Pushing the limits of mathematical reasoning in open language models},
  author={Shao, Zhihong and Wang, Peiyi and Zhu, Qihao and Xu, Runxin and Song, Junxiao and Bi, Xiao and Zhang, Haowei and Zhang, Mingchuan and Li, YK and Wu, Yang and others},
  journal={arXiv preprint arXiv:2402.03300},
  year={2024}
}

@article{hu2025reinforce++,
  title={Reinforce++: Stabilizing critic-free policy optimization with global advantage normalization},
  author={Hu, Jian and Liu, Jason Klein and Xu, Haotian and Shen, Wei},
  journal={arXiv preprint arXiv:2501.03262},
  year={2025}
}

@article{zheng2025group,
  title = {Group sequence policy optimization},
  author = {Zheng, Chujie and Liu, Shixuan and Li, Mingze and Chen, Xiong-Hui and Yu, Bowen and Gao, Chang and Dang, Kai and Liu, Yuqiong and Men, Rui and Yang, An and others},
  journal = {arXiv preprint arXiv:2507.18071},
  year = {2025}
}

@article{yu2026dapo,
  title={Dapo: An open-source llm reinforcement learning system at scale},
  author={Yu, Qiying and Zhang, Zheng and Zhu, Ruofei and Yuan, Yufeng and Zuo, Xiaochen and Yue, Yu and Dai, Weinan and Fan, Tiantian and Liu, Gaohong and Liu, Lingjun and others},
  journal={Advances in Neural Information Processing Systems},
  volume={38},
  pages={113222--113244},
  year={2026}
}

@inproceedings{wang2023selfconsistency,
    title={Self-Consistency Improves Chain of Thought Reasoning in Language Models},
    author={Xuezhi Wang and Jason Wei and Dale Schuurmans and Quoc V Le and Ed H. Chi and Sharan Narang and Aakanksha Chowdhery and Denny Zhou},
    booktitle={The Eleventh International Conference on Learning Representations },
    year={2023}
}

@article{zuo2026ttrl,
  title={Ttrl: Test-time reinforcement learning},
  author={Zuo, Yuxin and Zhang, Kaiyan and Sheng, Li and Qu, Shang and Cui, Ganqu and Zhu, Xuekai and Li, Haozhan and Long, Xinwei and Hua, Ermo and Qi, Biqing and others},
  journal={Advances in Neural Information Processing Systems},
  volume={38},
  pages={131459--131483},
  year={2026}
}

@inproceedings{wu2026srttrl,
    title={Beyond Majority Voting: Self-Reflective Test-Time Reinforcement Learning for {LLM} Reasoning},
    author={Sitong Wu and Haoru Tan and Xichen Zhang and Bin Xia and Shaofeng Zhang and Xiaojuan Qi and Bei Yu and Jiaya Jia},
    booktitle={Forty-third International Conference on Machine Learning},
    year={2026}
}

@inproceedings{wang2026beyond,
  title={Beyond majority voting: Towards fine-grained and more reliable reward signal for test-time reinforcement learning},
  author={Wang, Weiqin and Wang, Yile and Chen, Kehao and Huang, Hui},
  booktitle={Proceedings of the 64th Annual Meeting of the Association for Computational Linguistics (Volume 1: Long Papers)},
  pages={37251--37265},
  year={2026}
}

@inproceedings{wang2026selfharmony,
  title={Self-harmony: Learning to harmonize self-supervision and self-play in test-time reinforcement learning},
  author={Wang, Ru and Huang, Wei and Cao, Qi and Iwasawa, Yusuke and Matsuo, Yutaka and Guo, Jiaxian},
  booktitle={International Conference on Learning Representations},
  volume={2026},
  pages={83112--83144},
  year={2026}
}

@inproceedings{zhang2026corewarding,
  title={Co-rewarding: Stable self-supervised rl for eliciting reasoning in large language models},
  author={Zhang, Zizhuo and Zhu, Jianing and Ge, Xinmu and Zhao, Zihua and Li, Xuan and Feng, Xiao and Yao, Jiangchao and Han, Bo and others},
  booktitle={International Conference on Learning Representations},
  volume={2026},
  pages={82520--82558},
  year={2026}
}

@inproceedings{yu2026restrain,
    title={{RESTRAIN}: From Spurious Votes to Signals {\textemdash} Self-Training {RL} with Self-Penalization},
    author={ZHAONING YU and Zhaolun Su and Leitian Tao and Haozhu Wang and Aashu Singh and Hanchao Yu and Jianyu Wang and Hongyang Gao and Weizhe Yuan and Jason E Weston and Ping Yu and Jing Xu},
    booktitle={The Fourteenth International Conference on Learning Representations},
    year={2026}
}

@article{bai2022constitutional,
  title={Constitutional ai: Harmlessness from ai feedback},
  author={Bai, Yuntao and Kadavath, Saurav and Kundu, Sandipan and Askell, Amanda and Kernion, Jackson and Jones, Andy and Chen, Anna and Goldie, Anna and Mirhoseini, Azalia and McKinnon, Cameron and others},
  journal={arXiv preprint arXiv:2212.08073},
  year={2022}
}

@inproceedings{lee2024rlaif,
    title={{RLAIF} vs. {RLHF}: Scaling Reinforcement Learning from Human Feedback with {AI} Feedback},
    author={Harrison Lee and Samrat Phatale and Hassan Mansoor and Thomas Mesnard and Johan Ferret and Kellie Ren Lu and Colton Bishop and Ethan Hall and Victor Carbune and Abhinav Rastogi and Sushant Prakash},
    booktitle={Forty-first International Conference on Machine Learning},
    year={2024}
}

@article{wang2024selftaught,
  title={Self-taught evaluators},
  author={Wang, Tianlu and Kulikov, Ilia and Golovneva, Olga and Yu, Ping and Yuan, Weizhe and Dwivedi-Yu, Jane and Pang, Richard Yuanzhe and Fazel-Zarandi, Maryam and Weston, Jason and Li, Xian},
  journal={arXiv preprint arXiv:2408.02666},
  year={2024}
}

@inproceedings{wu2024metarewarding,
  title={Meta-rewarding language models: Self-improving alignment with llm-as-a-meta-judge},
  author={Wu, Tianhao and Yuan, Weizhe and Golovneva, Olga and Xu, Jing and Tian, Yuandong and Jiao, Jiantao and Weston, Jason E and Sukhbaatar, Sainbayar},
  booktitle={Proceedings of the 2025 Conference on Empirical Methods in Natural Language Processing},
  pages={11548--11565},
  year={2025}
}

@inproceedings{yuan2024selfrewarding,
  title = {Self-Rewarding Language Models},
  author = {Yuan, Weizhe and Pang, Richard Yuanzhe and Cho, Kyunghyun and Li, Xian and Sukhbaatar, Sainbayar and Xu, Jing and Weston, Jason E},
  booktitle = {Proceedings of the 41st International Conference on Machine Learning},
  pages = {57905--57923},
  year = {2024},
  editor = {Salakhutdinov, Ruslan and Kolter, Zico and Heller, Katherine and Weller, Adrian and Oliver, Nuria and Scarlett, Jonathan and Berkenkamp, Felix},
  volume = {235},
  series = {Proceedings of Machine Learning Research},
  month = {21--27 Jul},
  publisher = {PMLR}
}

@inproceedings{zhao2025learning,
  title={Learning to reason without external rewards},
  author={Zhao, Xuandong and Kang, Zhewei and Feng, Aosong and Levine, Sergey and Song, Dawn},
  booktitle={International Conference on Learning Representations},
  volume={2026},
  pages={2548--2581},
  year={2026}
}

@inproceedings{he2026how,
  title={How Far Can Unsupervised RLVR Scale LLM Training?},
  author={He, Bingxiang and Zuo, Yuxin and Liu, Zeyuan and Zhao, Shangziqi and Fu, Zixuan and Yang, Junlin and Qian, Cheng and Zhang, Kaiyan and Fan, Yuchen and Cui, Ganqu and others},
  booktitle={International Conference on Learning Representations},
  volume={2026},
  pages={14823--14865},
  year={2026}
}

@article{zhang2026consistent,
  title={Consistent paths lead to truth: Self-rewarding reinforcement learning for llm reasoning},
  author={Zhang, Kongcheng and Yao, Qi and Liu, Shunyu and Wang, Yingjie and Lai, Baisheng and Ye, Jieping and Song, Mingli and Tao, Dacheng},
  journal={Advances in Neural Information Processing Systems},
  volume={38},
  pages={59849--59887},
  year={2026}
}

@article{agarwal2026unreasonable,
  title={The unreasonable effectiveness of entropy minimization in llm reasoning},
  author={Agarwal, Shivam and Zhang, Zimin and Yuan, Lifan and Han, Jiawei and Peng, Hao},
  journal={Advances in Neural Information Processing Systems},
  volume={38},
  pages={107150--107180},
  year={2026}
}

@article{li2025confidence,
  title={Confidence is all you need: Few-shot rl fine-tuning of language models},
  author={Li, Pengyi and Skripkin, Matvey and Zubrey, Alexander and Kuznetsov, Andrey and Oseledets, Ivan},
  journal={arXiv preprint arXiv:2506.06395},
  year={2025}
}

@article{zhang2025no,
  title={No free lunch: Rethinking internal feedback for llm reasoning},
  author={Zhang, Yanzhi and Zhang, Zhaoxi and Guan, Haoxiang and Cheng, Yilin and Duan, Yitong and Wang, Chen and Wang, Yue and Zheng, Shuxin and He, Jiyan},
  journal={arXiv preprint arXiv:2506.17219},
  year={2025}
}

@article{roy2025you,
  title={You Need Reasoning to Learn Reasoning: The Limitations of Label-Free RL in Weak Base Models},
  author={Roy, Shuvendu and Hajimirsadeghi, Hossein and Zhai, Mengyao and Samei, Golnoosh},
  journal={arXiv preprint arXiv:2511.04902},
  year={2025}
}

@article{yang2025qwen3,
  title={Qwen3 technical report},
  author={Yang, An and Li, Anfeng and Yang, Baosong and Zhang, Beichen and Hui, Binyuan and Zheng, Bo and Yu, Bowen and Gao, Chang and Huang, Chengen and Lv, Chenxu and others},
  journal={arXiv preprint arXiv:2505.09388},
  year={2025}
}

@article{grattafiori2024llama,
  title={The llama 3 herd of models},
  author={Grattafiori, Aaron and Dubey, Abhimanyu and Jauhri, Abhinav and Pandey, Abhinav and Kadian, Abhishek and Al-Dahle, Ahmad and Letman, Aiesha and Mathur, Akhil and Schelten, Alan and Vaughan, Alex and others},
  journal={arXiv preprint arXiv:2407.21783},
  year={2024}
}

@inproceedings{hendrycks2021measuring,
    title={Measuring Mathematical Problem Solving With the {MATH} Dataset},
    author={Dan Hendrycks and Collin Burns and Saurav Kadavath and Akul Arora and Steven Basart and Eric Tang and Dawn Song and Jacob Steinhardt},
    booktitle={Thirty-fifth Conference on Neural Information Processing Systems Datasets and Benchmarks Track (Round 2)},
    year={2021},
}

@inproceedings{rein2023gpqa,
    title={{GPQA}: A Graduate-Level Google-Proof Q\&A Benchmark},
    author={David Rein and Betty Li Hou and Asa Cooper Stickland and Jackson Petty and Richard Yuanzhe Pang and Julien Dirani and Julian Michael and Samuel R. Bowman},
    booktitle={First Conference on Language Modeling},
    year={2024}
}

@article{wang2024mmlu,
  title={Mmlu-pro: A more robust and challenging multi-task language understanding benchmark},
  author={Wang, Yubo and Ma, Xueguang and Zhang, Ge and Ni, Yuansheng and Chandra, Abhranil and Guo, Shiguang and Ren, Weiming and Arulraj, Aaran and He, Xuan and Jiang, Ziyan and others},
  journal={Advances in Neural Information Processing Systems},
  volume={37},
  pages={95266--95290},
  year={2024}
}

@article{wang2026breaking,
  title={On the Overscaling Curse of Parallel Thinking: System Efficacy Contradicts Sample Efficiency},
  author={Wang, Yiming and Zhang, Zhuosheng and Wang, Rui},
  journal={arXiv preprint arXiv:2601.21619},
  year={2026}
}

@article{jensen1906,
  title={Sur les fonctions convexes et les in{\'e}galit{\'e}s entre les valeurs moyennes},
  author={Jensen, Johan Ludwig William Valdemar},
  journal={Acta mathematica},
  volume={30},
  number={1},
  pages={175--193},
  year={1906},
  publisher={Springer}
}

@inproceedings{jain2024livecodebench,
  title={Livecodebench: Holistic and contamination free evaluation of large language models for code},
  author={Jain, Naman and Gu, Alex and Li, Wen-Ding and Yan, Fanjia and Zhang, Tianjun and Wang, Sida and Solar-Lezama, Armando and Sen, Koushik and Stoica, Ion},
  booktitle={International Conference on Learning Representations},
  volume={2025},
  pages={58791--58831},
  year={2025}
}

@inproceedings{
    dang2025reinforcement,
    title={Reinforcement Learning for Reasoning in Small {LLM}s: What Works and What Doesn{\textquoteright}t},
    author={Quy-Anh Dang and Chris Ngo},
    booktitle={Logical and Symbolic Reasoning in Language Models @ AAAI 2026},
    year={2026}
}
\bibliographystyle{iclr2027_conference}

\appendix
\section{Additional Demonstrations of Self-Reward Uncertainty}
\label{app:more-motivation}

We provide four additional cases following the setup of
\textcolor{deepred}{Figure~\ref{fig:reward-stochasticity-violin}}.
For each \texttt{DAPO-14K} prompt, we sample one correct and one incorrect focal response from \texttt{Qwen3-8B-Base}.
Each focal response is evaluated under $M=100$ independent group contexts, each containing $G-1$ newly sampled responses with $G=16$.
We record its reward and group-normalized advantage under the four self-rewarding methods, together with its RLVR advantage under the same context.

\begin{figure}[H]
    \centering
    \includegraphics[width=\textwidth]{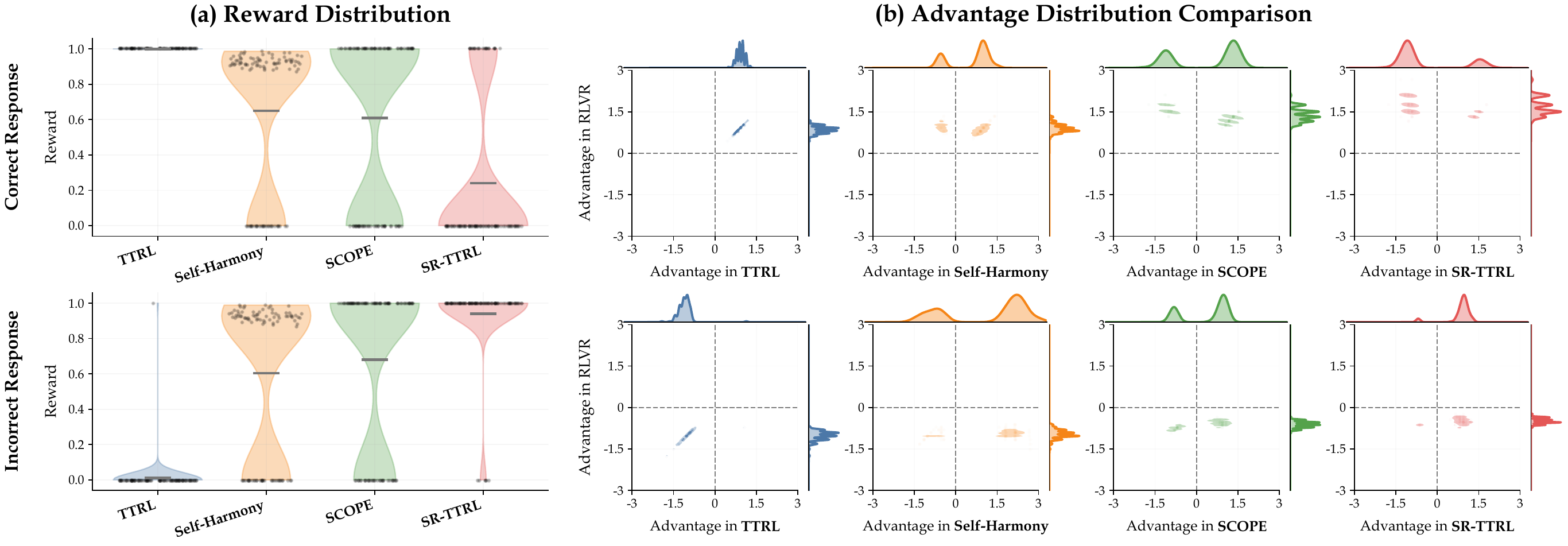}
    \vspace{-0.25in}
    \caption{\textbf{Self-Reward Uncertainty (Additional Case 2).}}
    \label{fig:more-motivation-case2}
\end{figure}

\begin{figure}[H]
    \centering
    \includegraphics[width=\textwidth]{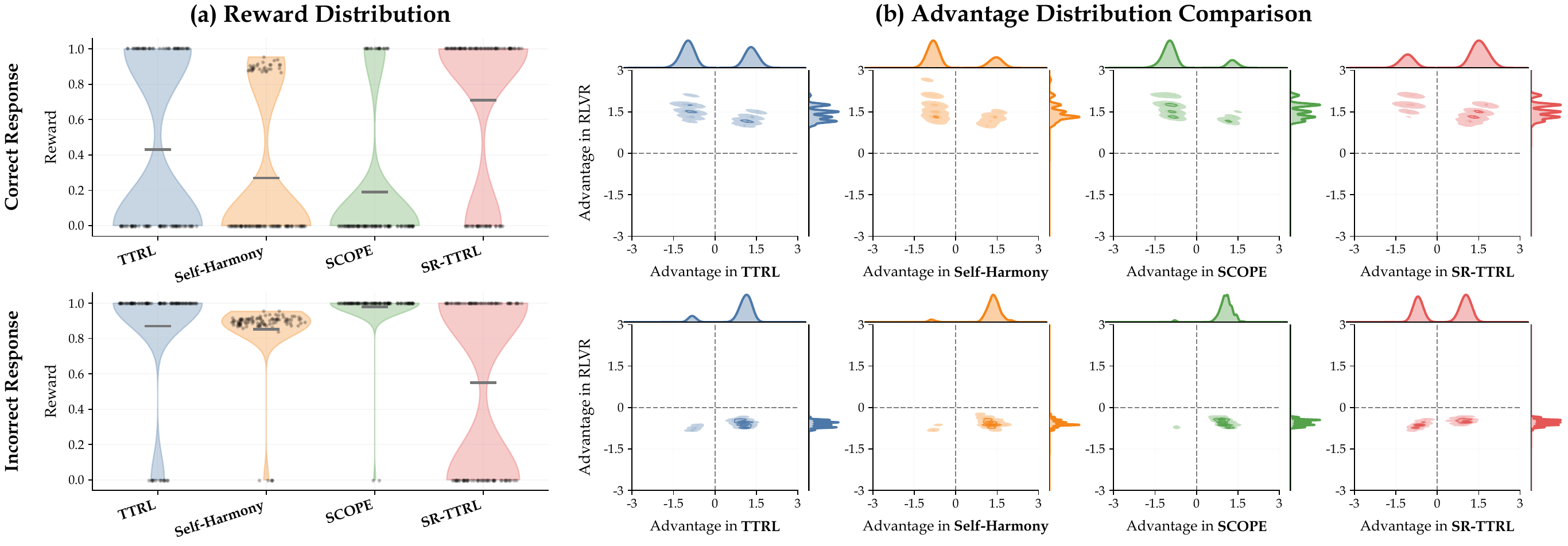}
    \vspace{-0.25in}
    \caption{\textbf{Self-Reward Uncertainty (Additional Case 3).}}
    \label{fig:more-motivation-case3}
\end{figure}

\begin{figure}[H]
    \centering
    \includegraphics[width=\textwidth]{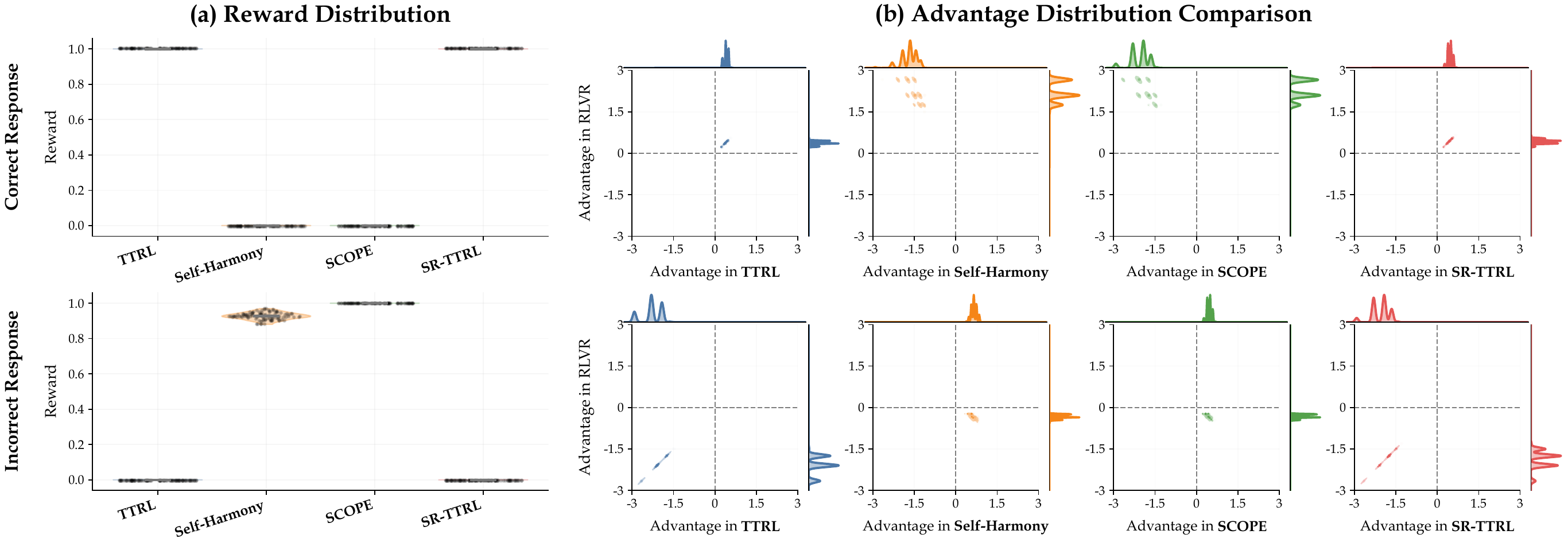}
    \vspace{-0.25in}
    \caption{\textbf{Self-Reward Uncertainty (Additional Case 4).}}
    \label{fig:more-motivation-case4}
\end{figure}

\begin{figure}[H]
    \centering
    \includegraphics[width=\textwidth]{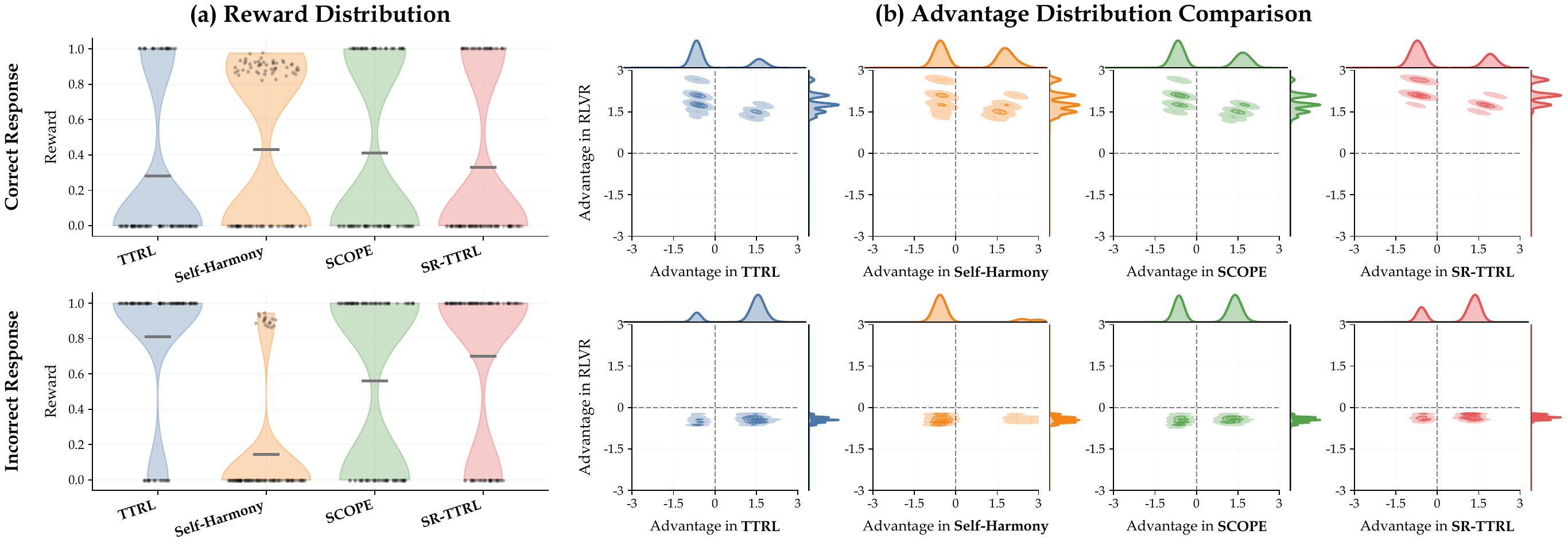}
    \vspace{-0.25in}
    \caption{\textbf{Self-Reward Uncertainty (Additional Case 5).}}
    \label{fig:more-motivation-case5}
\end{figure}
\section{Self-Rewarding Rules of Baselines}
\label{app:baseline-methods}

We describe the self-rewarding rules used by the baselines.
Following \textcolor{deepred}{Section~\ref{sec:preliminary}}, let
$\mathcal G^{\otimes G}=\{\mathbf y_1,\ldots,\mathbf y_G\}$
be the sampled response group, where $\mathbf y_i$ denotes a complete
response and $\mathcal A(\mathbf y_i)$ extracts its final answer.
Let
\begin{equation}
\mathcal U
=
\left\{
\mathcal A(\mathbf y_i)
\right\}_{i=1}^{G},
\qquad
n(a)
=
\sum_{i=1}^{G}
\mathbbm{1}\!\left[
\mathcal A(\mathbf y_i)=a
\right]
\end{equation}
denote the set of distinct answers and the frequency of answer $a$,
respectively.
In general, each method constructs a reward reference
$z=f(\mathcal G^{\otimes G})$ and assigns
$r_i=\mathcal R(\mathbf y_i\mid z)$.
Depending on the method, $z$ may be a single pseudo-label, a collection
of response-specific pseudo-labels, or structured group information.

\subsection{TTRL}

TTRL~\colorcitep{zuo2026ttrl} uses majority voting over the extracted
answers.
The most frequent answer is selected as the reward reference:
\begin{equation}
\boxed{
z_{\mathrm{TTRL}}
=
f_{\mathrm{TTRL}}(\mathcal G^{\otimes G})
=
\arg\max_{a\in\mathcal U}
\sum_{j=1}^{G}
\mathbbm{1}\!\left[
\mathcal A(\mathbf y_j)=a
\right]
}
\label{eq:app-ttrl-f}
\end{equation}
Each response receives a binary reward according to whether its final
answer matches this reference:
\begin{equation}
\boxed{
\mathcal R_{\mathrm{TTRL}}
\!\left(
\mathbf y_i\mid z_{\mathrm{TTRL}}
\right)
=
\mathbbm{1}\!\left[
\mathcal A(\mathbf y_i)=z_{\mathrm{TTRL}}
\right]
}
\label{eq:app-ttrl-r}
\end{equation}

\subsection{CoVo}

CoVo~\colorcitep{zhang2026consistent} derives rewards from the internal
reasoning trajectory of each response.
Specifically, response $\mathbf y_i$ is divided into $T_i$ reasoning
steps, and
$\mathbf s_{i,\ell}$ denotes its intermediate state before step $\ell$.
For a candidate answer $a$, CoVo measures its distance from this state by
\begin{equation}
d(\mathbf s_{i,\ell},a)
=
-\frac{1}{|a|}
\sum_{q=1}^{|a|}
\log
\pi_{\theta_{\mathrm{old}}}
\!\left(
a_q\mid\mathbf s_{i,\ell},a_{<q}
\right).
\label{eq:app-covo-distance}
\end{equation}
Let $a_i=\mathcal A(\mathbf y_i)$.
The consistency and volatility of $\mathbf y_i$ are defined as
\begin{equation}
\operatorname{Con}(\mathbf y_i)
=
\frac{1}{T_i}
\sum_{\ell=0}^{T_i-1}
\mathbbm{1}\!\left[
a_i
=
\arg\min_{a\in\mathcal U}
d(\mathbf s_{i,\ell},a)
\right],
\label{eq:app-covo-consistency}
\end{equation}
\begin{equation}
\operatorname{Vol}(\mathbf y_i)
=
\frac{1}{T_i}
\max\!\left(
\left\{
\ell:
a_i
\neq
\arg\min_{a\in\mathcal U}
d(\mathbf s_{i,\ell},a)
\right\}
\cup\{0\}
\right).
\label{eq:app-covo-volatility}
\end{equation}

Responses sharing the same extracted answer form
\begin{equation}
\mathcal I(a)
=
\left\{
j:
\mathcal A(\mathbf y_j)=a
\right\}.
\end{equation}
CoVo maps each response to
\begin{equation}
\mathbf v_i
=
\operatorname{Con}(\mathbf y_i)
\begin{bmatrix}
\cos\!\left(\operatorname{Vol}(\mathbf y_i)\right)\\
\sin\!\left(\operatorname{Vol}(\mathbf y_i)\right)
\end{bmatrix}
\end{equation}
and aggregates these vectors within each answer group.
Accordingly, its group-derived reference is
\begin{equation}
\boxed{
z_{\mathrm{CoVo}}
=
f_{\mathrm{CoVo}}(\mathcal G^{\otimes G})
=
\left\{
\mathcal I(a),
\left\{
\mathbf v_j
\right\}_{j\in\mathcal I(a)}
\right\}_{a\in\mathcal U}
}
\label{eq:app-covo-f}
\end{equation}
The intrinsic reward of $\mathbf y_i$ is the normalized magnitude of the
aggregated vector for its answer group:
\begin{equation}
r_i^{\mathrm{int}}
=
\frac{1}{|\mathcal I(a_i)|}
\left\|
\sum_{j\in\mathcal I(a_i)}
\mathbf v_j
\right\|_2.
\label{eq:app-covo-intrinsic}
\end{equation}
CoVo additionally assigns a curiosity reward
$r_i^{\mathrm{cur}}$ based on the likelihood of transitions between
successive reasoning states, together with a KL penalty that prevents a
few extremely unlikely tokens from dominating the reward.
The final reward is
\begin{equation}
\boxed{
\mathcal R_{\mathrm{CoVo}}
\!\left(
\mathbf y_i\mid z_{\mathrm{CoVo}}
\right)
=
r_i^{\mathrm{int}}
+
r_i^{\mathrm{cur}}
}
\label{eq:app-covo-r}
\end{equation}
The token-level computation of $r_i^{\mathrm{cur}}$ follows the original
implementation of CoVo.

\subsection{SCOPE}

SCOPE~\colorcitep{wang2026beyond} combines step-wise confidence weighting
with subgroup-specific pseudo-labels.
Suppose response $\mathbf y_i$ contains $L_i$ reasoning steps, with step
$\ell$ containing $N_{i,\ell}$ tokens.
Let
$P_{i,\ell,t}^{(1)},\ldots,P_{i,\ell,t}^{(K_c)}$
denote the top-$K_c$ next-token probabilities at token position $t$.
The response-level confidence is
\begin{equation}
c_i
=
\frac{1}{L_i}
\sum_{\ell=1}^{L_i}
\frac{1}{N_{i,\ell}}
\sum_{t=1}^{N_{i,\ell}}
\left(
-\frac{1}{K_c}
\sum_{q=1}^{K_c}
\log P_{i,\ell,t}^{(q)}
\right).
\label{eq:app-scope-confidence}
\end{equation}

SCOPE partitions the $G$ responses into equal-sized subgroups.
The subgroup size is dynamically selected through the Pareto criterion
in the original method, which balances the agreement between responses
and their subgroup labels against the diversity of labels across
subgroups.
Let $\mathcal I_1,\ldots,\mathcal I_L$ denote the resulting partition and
let $\ell(i)$ satisfy $i\in\mathcal I_{\ell(i)}$.

For each subgroup $\ell$, SCOPE independently draws a bootstrap sample
$\mathcal B_\ell$ from the complete rollout group and applies
confidence-weighted voting:
\begin{equation}
z_\ell
=
\arg\max_{a\in\mathcal U}
\sum_{j\in\mathcal B_\ell}
c_j
\mathbbm{1}\!\left[
\mathcal A(\mathbf y_j)=a
\right].
\label{eq:app-scope-subgroup-label}
\end{equation}
Therefore, SCOPE constructs one reward reference for each response
according to its subgroup:
\begin{equation}
\boxed{
z_{\mathrm{SCOPE}}
=
f_{\mathrm{SCOPE}}(\mathcal G^{\otimes G})
=
\left(
z_{\ell(1)},\ldots,z_{\ell(G)}
\right)
}
\label{eq:app-scope-f}
\end{equation}
Writing $z_{\mathrm{SCOPE},i}=z_{\ell(i)}$, its reward is
\begin{equation}
\boxed{
\mathcal R_{\mathrm{SCOPE}}
\!\left(
\mathbf y_i\mid z_{\mathrm{SCOPE}}
\right)
=
\mathbbm{1}\!\left[
\mathcal A(\mathbf y_i)=z_{\mathrm{SCOPE},i}
\right]
}
\label{eq:app-scope-r}
\end{equation}
The candidate subgroup sizes, bootstrap configuration, and Pareto
selection procedure follow the original SCOPE implementation.

\subsection{Self-Harmony}

Self-Harmony~\colorcitep{wang2026selfharmony} constructs two views of the
same prompt.
Besides the original prompt $\mathbf x$, the model acts as a Reframer to
generate a semantically equivalent prompt $\mathbf x'$.
It then samples an additional response group
\begin{equation}
\mathcal G'^{\otimes G'}
=
\left\{
\mathbf y'_1,\ldots,\mathbf y'_{G'}
\right\}
\sim
\pi_{\theta_{\mathrm{old}}}^{\otimes G'}
(\cdot\mid\mathbf x').
\end{equation}
For each candidate answer $a$, its empirical frequencies in the original
and reframed views are
\begin{equation}
p(a)
=
\frac{1}{G}
\sum_{i=1}^{G}
\mathbbm{1}\!\left[
\mathcal A(\mathbf y_i)=a
\right],
\qquad
p'(a)
=
\frac{1}{G'}
\sum_{j=1}^{G'}
\mathbbm{1}\!\left[
\mathcal A(\mathbf y'_j)=a
\right].
\end{equation}
Self-Harmony selects the answer with the largest harmonic mean of these
two frequencies:
\begin{equation}
\boxed{
z_{\mathrm{SH}}
=
f_{\mathrm{SH}}
\!\left(
\mathcal G^{\otimes G},
\mathcal G'^{\otimes G'}
\right)
=
\arg\max_a
\frac{2p(a)p'(a)}
{p(a)+p'(a)}
}
\label{eq:app-self-harmony-f}
\end{equation}
The score is defined as zero when $p(a)+p'(a)=0$.
The original responses are rewarded by agreement with this shared
reference:
\begin{equation}
\boxed{
\mathcal R_{\mathrm{SH}}
\!\left(
\mathbf y_i\mid z_{\mathrm{SH}}
\right)
=
\mathbbm{1}\!\left[
\mathcal A(\mathbf y_i)=z_{\mathrm{SH}}
\right]
}
\label{eq:app-self-harmony-r}
\end{equation}
Self-Harmony also trains the reframing branch using format and diversity
rewards.
Their prompt templates and implementation details follow the original
paper.

\subsection{RESTRAIN}

RESTRAIN~\colorcitep{yu2026restrain} retains all distinct answers instead
of selecting only the majority answer.
Write
$\mathcal U=\{a_1,\ldots,a_M\}$ and define
\begin{equation}
c_m
=
\sum_{i=1}^{G}
\mathbbm{1}\!\left[
\mathcal A(\mathbf y_i)=a_m
\right],
\qquad
p_m=\frac{c_m}{G}.
\end{equation}
Given the Gaussian shaping function
\begin{equation}
g(p)
=
\exp\!\left(
-\frac{(p-k)^2}{2\sigma^2}
\right),
\qquad
k\in[0,1],
\quad
\sigma>0,
\end{equation}
RESTRAIN assigns each candidate answer the normalized weight
\begin{equation}
w_m
=
\frac{g(p_m)}
{\sum_{\ell=1}^{M}g(p_\ell)}.
\end{equation}
Its reward reference is therefore a weighted collection of candidate
answers:
\begin{equation}
\boxed{
z_{\mathrm{RES}}
=
f_{\mathrm{RES}}(\mathcal G^{\otimes G})
=
\left\{
(a_m,w_m)
\right\}_{m=1}^{M}
}
\label{eq:app-restrain-f}
\end{equation}

Unlike methods using a single pseudo-label, RESTRAIN evaluates each
response against every candidate answer:
\begin{equation}
\boxed{
\mathcal R_{\mathrm{RES}}
\!\left(
\mathbf y_i\mid z_{\mathrm{RES}}
\right)
=
\left\{
\mathbbm{1}\!\left[
\mathcal A(\mathbf y_i)=a_m
\right]
\right\}_{m=1}^{M}
}
\label{eq:app-restrain-r}
\end{equation}
Let
$r_{i,m}=\mathbbm{1}[\mathcal A(\mathbf y_i)=a_m]$
and let $A_{i,m}$ be the group-normalized advantage computed from
$\{r_{j,m}\}_{j=1}^{G}$.
When the majority count
$c_{\max}=\max_m c_m$
falls below a threshold $\kappa$, RESTRAIN sets the rewards to zero and
subtracts a fixed offset $\delta$ from every advantage:
\begin{equation}
\widetilde r_{i,m}
=
\begin{cases}
r_{i,m}, & c_{\max}\geq\kappa,\\
0, & c_{\max}<\kappa,
\end{cases}
\qquad
\widetilde A_{i,m}
=
\begin{cases}
A_{i,m}, & c_{\max}\geq\kappa,\\
A_{i,m}-\delta, & c_{\max}<\kappa.
\end{cases}
\label{eq:app-restrain-penalty}
\end{equation}
The resulting GRPO objectives are combined using the pseudo-label weights:
\begin{equation}
\mathcal J_{\mathrm{RES}}(\theta)
=
u_{\mathbf x}
\sum_{m=1}^{M}
w_m
\mathcal J_{\mathrm{GRPO}}
\!\left(
\theta;\{\widetilde A_{i,m}\}_{i=1}^{G}
\right),
\label{eq:app-restrain-objective}
\end{equation}
where $u_{\mathbf x}$ is a fixed prompt weight obtained from the
self-consistency of a frozen reference model, as defined in the original
RESTRAIN implementation.

\subsection{Co-Rewarding}

Co-Rewarding~\colorcitep{zhang2026corewarding} constructs its reward
reference using a slowly evolving teacher policy.
At training step $t$, the teacher parameters are updated by
\begin{equation}
\widetilde{\theta}_t
=
\alpha_t\widetilde{\theta}_{t-1}
+
(1-\alpha_t)\theta_{\mathrm{old},t},
\end{equation}
where $\alpha_t\in[0,1]$ is the EMA coefficient.
The teacher samples
\begin{equation}
\widetilde{\mathcal G}^{\otimes\widetilde G}
=
\left\{
\widetilde{\mathbf y}_1,\ldots,
\widetilde{\mathbf y}_{\widetilde G}
\right\}
\sim
\pi_{\widetilde{\theta}_t}^{\otimes\widetilde G}
(\cdot\mid\mathbf x).
\end{equation}
The majority answer in this teacher-generated group becomes the reward
reference:
\begin{equation}
\boxed{
z_{\mathrm{CR}}
=
f_{\mathrm{CR}}
\!\left(
\widetilde{\mathcal G}^{\otimes\widetilde G}
\right)
=
\arg\max_a
\sum_{j=1}^{\widetilde G}
\mathbbm{1}\!\left[
\mathcal A(\widetilde{\mathbf y}_j)=a
\right]
}
\label{eq:app-corewarding-f}
\end{equation}
The student responses are evaluated against this reference:
\begin{equation}
\boxed{
\mathcal R_{\mathrm{CR}}
\!\left(
\mathbf y_i\mid z_{\mathrm{CR}}
\right)
=
\mathbbm{1}\!\left[
\mathcal A(\mathbf y_i)=z_{\mathrm{CR}}
\right]
}
\label{eq:app-corewarding-r}
\end{equation}

\subsection{SR-TTRL}
\label{app:baseline-srttrl}

SR-TTRL~\colorcitep{wu2026srttrl} constructs its pseudo-label through
three stages: exploration, summarization, and reflection.
Given the sampled group
$\mathcal G^{\otimes G}=\{\mathbf y_1,\ldots,\mathbf y_G\}$,
the exploration stage first collects the distinct extracted answers
\begin{equation}
\mathcal U
=
\left\{
\mathcal A(\mathbf y_j)
\right\}_{j=1}^{G}.
\end{equation}
For each $a\in\mathcal U$,
$\operatorname{Rep}$ selects a representative response from those
producing answer $a$.
The summarization operator
$\operatorname{Summ}_{\pi_{\theta_{\mathrm{old}}}}$
then uses the current prompt and this representative response to produce
a concise reasoning summary $\mathbf s_a$.
Finally,
$\operatorname{Reflect}_{\pi_{\theta_{\mathrm{old}}}}$
jointly compares all answer--summary pairs and returns the selected answer
as the pseudo-label:

\begin{equation}
\begin{aligned}
&\mathbf s_a
=
\operatorname{Summ}_{\pi_{\theta_{\mathrm{old}}}}
\left(
\mathbf x,
\operatorname{Rep}
\left\{
\mathbf y_j
\mid
\mathcal A(\mathbf y_j)=a
\right\}
\right),
\qquad
a\in\mathcal U,
\qquad
\mathcal U
=
\{\mathcal A(\mathbf y_j)\}_{j=1}^{G},
\\
&\boxed{
z_{\mathrm{SR}}
=
f_{\mathrm{SR}}(\mathcal G^{\otimes G})
=
\operatorname{Reflect}_{\pi_{\theta_{\mathrm{old}}}}
\left(
\mathbf x,
\{(a,\mathbf s_a)\}_{a\in\mathcal U}
\right)
}.
\end{aligned}
\label{eq:app-srttrl-f}
\end{equation}

Here, $\operatorname{Rep}$ denotes the representative-response selection
procedure, $\operatorname{Summ}$ summarizes the reasoning supporting each
candidate answer, and $\operatorname{Reflect}$ compares these summarized
reasoning paths to select the final reward reference.
The exact summarization and reflection prompts follow the original
SR-TTRL implementation.

Each sampled response is then evaluated against this reference using a
binary reward:
\begin{equation}
\boxed{
\mathcal R_{\mathrm{SR}}
\left(
\mathbf y_i
\mid
z_{\mathrm{SR}}
\right)
=
\mathbbm{1}
\left[
\mathcal A(\mathbf y_i)=z_{\mathrm{SR}}
\right]
}.
\label{eq:app-srttrl-r}
\end{equation}
\section{Theory of GMAE}

\subsection{Approximation Error of Shared-Pool Estimation}
\label{app:gmae-approximation-bound}

Here, we formalize the approximation error between the ideal
group-marginalized reward distribution $P_{r_i}$ in
\textcolor{deepred}{Eq.\ref{eq:group-marginalized-reward}}
and its shared-pool estimate $\widehat P_{r_i}$ in
\textcolor{deepred}{Eq.\ref{eq:gmae-empirical-reward}}.
Let $d_{\mathrm{TV}}$ denote total variation distance and assume that
$\mathcal V_i=\operatorname{supp}(P_{r_i})$ is finite.

\begin{theorem}[Approximation Error of Shared-Pool Estimation]
\label{thm:gmae-approximation}
For any $\delta\in(0,1)$, with probability at least $1-\delta$ over the
candidate-pool and context sampling,
\begin{equation}
d_{\mathrm{TV}}\!\left(
\widehat P_{r_i},P_{r_i}
\right)
\le
\sqrt{
|\mathcal V_i|\log 2+\log(4/\delta)
}
\left(
\frac{1}{\sqrt{2K}}
+
\sqrt{
\frac{G-1}{G+G'-1}
}
\right).
\label{eq:gmae-approximation-bound}
\end{equation}
\end{theorem}

\begin{proof}
Condition on the fixed response $\mathbf y_i$ and define
\begin{equation}
N=G+G'-1,
\qquad
m=G-1,
\qquad
q=\left\lfloor\frac{N}{m}\right\rfloor.
\label{eq:approx-proof-notation}
\end{equation}
Let $\mathbf z_1,\ldots,\mathbf z_N$ denote the remaining candidate
responses and
\begin{equation}
\mathfrak S_{N,m}
=
\left\{
S\subseteq\{1,\ldots,N\}:|S|=m
\right\}
\end{equation}
denote the set of all valid companion-index subsets.
For each $S\in\mathfrak S_{N,m}$, define the resulting reward as
\begin{equation}
\rho_i(S)
=
\mathcal R\!\left(
\mathbf y_i
\mid
f\!\left(
\{\mathbf y_i\}\cup
\{\mathbf z_j:j\in S\}
\right)
\right).
\end{equation}

Averaging over all contexts in the candidate pool gives the complete-pool
distribution
\begin{equation}
\overline P_{r_i}^{\mathcal C}
=
\binom{N}{m}^{-1}
\sum_{S\in\mathfrak S_{N,m}}
\delta_{\rho_i(S)}.
\label{eq:complete-pool-distribution}
\end{equation}
If $S_1,\ldots,S_K$ are sampled uniformly without replacement from
$\mathfrak S_{N,m}$, the empirical distribution used by GMAE is
\begin{equation}
\widehat P_{r_i}
=
\frac{1}{K}
\sum_{k=1}^{K}
\delta_{\rho_i(S_k)}.
\label{eq:sampled-pool-distribution}
\end{equation}

By the triangle inequality,
\begin{equation}
\begin{aligned}
d_{\mathrm{TV}}\!\left(
\widehat P_{r_i},P_{r_i}
\right)
\le\;&
d_{\mathrm{TV}}\!\left(
\widehat P_{r_i},
\overline P_{r_i}^{\mathcal C}
\right)
\\
&+
d_{\mathrm{TV}}\!\left(
\overline P_{r_i}^{\mathcal C},
P_{r_i}
\right).
\end{aligned}
\label{eq:approx-error-decomposition}
\end{equation}
The first term is the context-subsampling error caused by evaluating only
$K$ contexts, while the second is the finite-pool error caused by using
$N$ sampled candidate responses.

For distributions supported on the finite space $\mathcal V_i$,
\begin{equation}
d_{\mathrm{TV}}(P,Q)
=
\sup_{\mathcal A\subseteq\mathcal V_i}
\left|
P(\mathcal A)-Q(\mathcal A)
\right|.
\label{eq:tv-finite-support}
\end{equation}

We first bound the context-subsampling error.
For any $\mathcal A\subseteq\mathcal V_i$, define
\begin{equation}
X_k^{\mathcal A}
=
\mathbbm{1}\!\left[
\rho_i(S_k)\in\mathcal A
\right].
\end{equation}
Conditioned on the candidate pool,
$\widehat P_{r_i}(\mathcal A)$ is the mean of $K$ values sampled without
replacement from the finite population
$\{\mathbbm{1}[\rho_i(S)\in\mathcal A]:
S\in\mathfrak S_{N,m}\}$, whose population mean is
$\overline P_{r_i}^{\mathcal C}(\mathcal A)$.
Hoeffding's inequality for sampling without replacement therefore gives
\begin{equation}
\Pr\!\left[
\left|
\widehat P_{r_i}(\mathcal A)
-
\overline P_{r_i}^{\mathcal C}(\mathcal A)
\right|
\ge \epsilon_1
\,\middle|\,
\mathcal C
\right]
\le
2\exp\!\left(-2K\epsilon_1^2\right).
\end{equation}
Applying a union bound over the at most $2^{|\mathcal V_i|}$ subsets of
$\mathcal V_i$ yields
\begin{equation}
\Pr\!\left[
d_{\mathrm{TV}}\!\left(
\widehat P_{r_i},
\overline P_{r_i}^{\mathcal C}
\right)
\ge \epsilon_1
\right]
\le
2^{|\mathcal V_i|+1}
\exp\!\left(-2K\epsilon_1^2\right).
\label{eq:context-subsampling-bound}
\end{equation}
The bound is unconditional because its right-hand side does not depend
on the realization of $\mathcal C$.

We next bound the finite-pool error.
For each $\mathcal A\subseteq\mathcal V_i$, define the symmetric kernel
\begin{equation}
h_{\mathcal A}
(\mathbf z_1,\ldots,\mathbf z_m)
=
\mathbbm{1}\!\left[
\mathcal R\!\left(
\mathbf y_i
\mid
f\!\left(
\{\mathbf y_i,\mathbf z_1,\ldots,\mathbf z_m\}
\right)
\right)
\in\mathcal A
\right].
\end{equation}
Its population mean is
\begin{equation}
p_{\mathcal A}
=
\mathbb E\!\left[
h_{\mathcal A}
(\mathbf z_1,\ldots,\mathbf z_m)
\right]
=
P_{r_i}(\mathcal A).
\label{eq:kernel-population-mean}
\end{equation}
Moreover,
\begin{equation}
\overline P_{r_i}^{\mathcal C}(\mathcal A)
=
\binom{N}{m}^{-1}
\sum_{S\in\mathfrak S_{N,m}}
h_{\mathcal A}(\mathbf z_j:j\in S),
\label{eq:complete-pool-u-statistic}
\end{equation}
which is a U-statistic of order $m$.

To derive its concentration, let $\Pi_N$ be the set of all permutations
of $\{1,\ldots,N\}$.
For each $\pi\in\Pi_N$, divide its first $qm$ indices into $q$ disjoint
blocks of size $m$ and define
\begin{equation}
V_{\pi,\mathcal A}
=
\frac{1}{q}
\sum_{\ell=1}^{q}
h_{\mathcal A}\!\left(
\mathbf z_{\pi((\ell-1)m+1)},
\ldots,
\mathbf z_{\pi(\ell m)}
\right).
\end{equation}
Every $m$-subset appears equally often among these blocks over all
permutations, so
\begin{equation}
\overline P_{r_i}^{\mathcal C}(\mathcal A)
=
\frac{1}{N!}
\sum_{\pi\in\Pi_N}
V_{\pi,\mathcal A}.
\label{eq:u-statistic-permutation-average}
\end{equation}

For any $\lambda>0$, Jensen's inequality gives
\begin{equation}
\begin{aligned}
&\mathbb E\exp\!\left[
\lambda\left(
\overline P_{r_i}^{\mathcal C}(\mathcal A)
-p_{\mathcal A}
\right)
\right]
\\
&\quad\le
\frac{1}{N!}
\sum_{\pi\in\Pi_N}
\mathbb E\exp\!\left[
\lambda\left(
V_{\pi,\mathcal A}-p_{\mathcal A}
\right)
\right].
\end{aligned}
\label{eq:u-statistic-jensen}
\end{equation}
For a fixed permutation, the $q$ blocks are disjoint and hence their
kernel evaluations are independent variables in $[0,1]$, each with mean
$p_{\mathcal A}$.
Hoeffding's lemma therefore implies
\begin{equation}
\mathbb E\exp\!\left[
\lambda\left(
V_{\pi,\mathcal A}-p_{\mathcal A}
\right)
\right]
\le
\exp\!\left(
\frac{\lambda^2}{8q}
\right).
\end{equation}
Combining this with \textcolor{deepred}{Eq.\ref{eq:u-statistic-jensen}} and applying the
Chernoff bound gives
\begin{equation}
\Pr\!\left[
\overline P_{r_i}^{\mathcal C}(\mathcal A)
-p_{\mathcal A}
\ge\epsilon_2
\right]
\le
\inf_{\lambda>0}
\exp\!\left(
-\lambda\epsilon_2+\frac{\lambda^2}{8q}
\right)
=
\exp\!\left(-2q\epsilon_2^2\right),
\end{equation}
where the minimum is attained at $\lambda=4q\epsilon_2$.
Applying the same argument to the lower tail yields
\begin{equation}
\Pr\!\left[
\left|
\overline P_{r_i}^{\mathcal C}(\mathcal A)
-
P_{r_i}(\mathcal A)
\right|
\ge\epsilon_2
\right]
\le
2\exp\!\left(-2q\epsilon_2^2\right).
\end{equation}
A union bound over all subsets of $\mathcal V_i$ then gives
\begin{equation}
\Pr\!\left[
d_{\mathrm{TV}}\!\left(
\overline P_{r_i}^{\mathcal C},
P_{r_i}
\right)
\ge\epsilon_2
\right]
\le
2^{|\mathcal V_i|+1}
\exp\!\left(-2q\epsilon_2^2\right).
\label{eq:finite-pool-bound}
\end{equation}

Let
\begin{equation}
L_\delta
=
|\mathcal V_i|\log 2+\log(4/\delta)
\end{equation}
and choose
\begin{equation}
\epsilon_1
=
\sqrt{\frac{L_\delta}{2K}},
\qquad
\epsilon_2
=
\sqrt{\frac{L_\delta}{2q}}.
\end{equation}
By \textcolor{deepred}{Eq.\ref{eq:context-subsampling-bound}} and
\textcolor{deepred}{Eq.\ref{eq:finite-pool-bound}}, each error exceeds its corresponding bound
with probability at most $\delta/2$.
Thus, with probability at least $1-\delta$,
\begin{equation}
d_{\mathrm{TV}}\!\left(
\widehat P_{r_i},P_{r_i}
\right)
\le
\sqrt{L_\delta}
\left(
\frac{1}{\sqrt{2K}}
+
\frac{1}{\sqrt{2q}}
\right).
\label{eq:approx-bound-with-q}
\end{equation}

Finally, since $N/m\ge1$,
\begin{equation}
q
=
\left\lfloor\frac{N}{m}\right\rfloor
\ge
\frac{N}{2m},
\end{equation}
and therefore
\begin{equation}
\frac{1}{\sqrt{2q}}
\le
\sqrt{\frac{m}{N}}
=
\sqrt{\frac{G-1}{G+G'-1}}.
\end{equation}
Substituting this inequality into
\textcolor{deepred}{Eq.\ref{eq:approx-bound-with-q}} proves
\textcolor{deepred}{Eq.\ref{eq:gmae-approximation-bound}}.
\end{proof}

\subsection{Uncertainty-Aware Advantage Scale}
\label{app:advantage-uncertainty}

For prompt $b$, let
\begin{equation}
\widehat P_b
=
\bigotimes_{i=1}^{G}\widehat P_{r_{b,i}}
\label{eq:prompt-joint-reward-distribution}
\end{equation}
denote the joint reward distribution.
For a reward realization
$\mathbf v=(v_1,\ldots,v_G)\sim\widehat P_b$, define its normalized
advantage vector as
\begin{equation}
\mathbf a_b(\mathbf v)
=
\begin{cases}
\left(
\dfrac{v_1-\mu_{\mathbf v}}{\sigma_{\mathbf v}},
\ldots,
\dfrac{v_G-\mu_{\mathbf v}}{\sigma_{\mathbf v}}
\right),
& \sigma_{\mathbf v}>0,\\[8pt]
\mathbf 0,
& \sigma_{\mathbf v}=0.
\end{cases}
\label{eq:realization-level-advantage}
\end{equation}
The corresponding GMAE advantage vector is
\begin{equation}
\mathbf A_b^{\mathrm{gmae}}
=
\mathbb E_{\mathbf v}
\left[
\mathbf a_b(\mathbf v)
\right].
\label{eq:gmae-advantage-vector}
\end{equation}
Since every $\mathbf a_b(\mathbf v)$ is zero-centered,
$\sum_{i=1}^{G}A_{b,i}^{\mathrm{gmae}}=0$, and the prompt-wise GMAE
scale satisfies
\begin{equation}
\left(s_b^{\mathrm{gmae}}\right)^2
=
\operatorname{Var}
\left(
\{A_{b,i}^{\mathrm{gmae}}\}_{i=1}^{G}
\right)
=
\frac{1}{G}
\left\|
\mathbf A_b^{\mathrm{gmae}}
\right\|_2^2.
\label{eq:prompt-gmae-scale}
\end{equation}

Let $\mathbf v$ and $\mathbf v'$ be independent draws from
$\widehat P_b$.
We define the context-induced reward uncertainty as
\begin{equation}
\mathcal U_b
=
\Pr_{\mathbf v}\!\left(\sigma_{\mathbf v}=0\right)
+
\frac{1}{2G}
\mathbb E_{\mathbf v,\mathbf v'}
\left[
\left\|
\mathbf a_b(\mathbf v)
-
\mathbf a_b(\mathbf v')
\right\|_2^2
\right].
\label{eq:reward-uncertainty}
\end{equation}
The first term captures reward realizations with no relative reward
information, while the second measures disagreement between relative
reward assignments across realizations.

\begin{theorem}[Reward Uncertainty and GMAE Scale]
\label{thm:reward-uncertainty-scale}
For every prompt $b$,
\begin{equation}
\left(s_b^{\mathrm{gmae}}\right)^2
=
1-\mathcal U_b.
\label{eq:uncertainty-scale-identity}
\end{equation}
Consequently, greater context-induced reward uncertainty yields a
smaller prompt-wise GMAE scale.
\end{theorem}

\begin{proof}
For any non-degenerate reward realization, group-wise normalization gives
\begin{equation}
\frac{1}{G}
\left\|
\mathbf a_b(\mathbf v)
\right\|_2^2
=
\frac{1}{G}
\sum_{i=1}^{G}
\left(
\frac{v_i-\mu_{\mathbf v}}
{\sigma_{\mathbf v}}
\right)^2
=
1.
\label{eq:nondegenerate-realization-scale}
\end{equation}
When $\sigma_{\mathbf v}=0$, we assign
$\mathbf a_b(\mathbf v)=\mathbf 0$.
Therefore,
\begin{equation}
\frac{1}{G}
\mathbb E_{\mathbf v}
\left[
\left\|
\mathbf a_b(\mathbf v)
\right\|_2^2
\right]
=
1-
\Pr_{\mathbf v}\!\left(
\sigma_{\mathbf v}=0
\right).
\label{eq:expected-realization-scale}
\end{equation}

Because $\mathbf v$ and $\mathbf v'$ are independent and identically
distributed,
\begin{equation}
\begin{aligned}
&
\frac{1}{2G}
\mathbb E_{\mathbf v,\mathbf v'}
\left[
\left\|
\mathbf a_b(\mathbf v)
-
\mathbf a_b(\mathbf v')
\right\|_2^2
\right]
\\
&\quad=
\frac{1}{G}
\mathbb E_{\mathbf v}
\left[
\left\|
\mathbf a_b(\mathbf v)
\right\|_2^2
\right]
-
\frac{1}{G}
\left\|
\mathbb E_{\mathbf v}
\left[
\mathbf a_b(\mathbf v)
\right]
\right\|_2^2
\\
&\quad=
1-
\Pr_{\mathbf v}\!\left(
\sigma_{\mathbf v}=0
\right)
-
\left(s_b^{\mathrm{gmae}}\right)^2.
\end{aligned}
\label{eq:pairwise-disagreement}
\end{equation}
Rearranging \textcolor{deepred}{Eq.\ref{eq:pairwise-disagreement}} gives
\begin{equation}
\begin{aligned}
\left(s_b^{\mathrm{gmae}}\right)^2
&=
1-
\Bigg[
\Pr_{\mathbf v}\!\left(
\sigma_{\mathbf v}=0
\right)
\\
&\qquad+
\frac{1}{2G}
\mathbb E_{\mathbf v,\mathbf v'}
\left[
\left\|
\mathbf a_b(\mathbf v)
-
\mathbf a_b(\mathbf v')
\right\|_2^2
\right]
\Bigg]
\\
&=
1-\mathcal U_b,
\end{aligned}
\label{eq:uncertainty-scale-proof}
\end{equation}
which proves the theorem.
\end{proof}

Theorem~\ref{thm:reward-uncertainty-scale} explains why advantage signs
matter.
When a response is favored in some reward realizations but disfavored in
others, its normalized advantage changes sign, increasing the
disagreement term in \textcolor{deepred}{Eq.\ref{eq:reward-uncertainty}}.
Marginalization then causes stronger cancellation and produces a smaller
scale, leading to more conservative policy updates.

Prompt-level calibration would independently restore every prompt to its
non-marginalized scale and remove these uncertainty-dependent
differences.
In contrast, batch-level calibration applies one shared factor
$\lambda$, so
\begin{equation}
\frac{
\widetilde s_b^{\mathrm{gmae}}
}{
\widetilde s_{b'}^{\mathrm{gmae}}
}
=
\frac{
\lambda s_b^{\mathrm{gmae}}
}{
\lambda s_{b'}^{\mathrm{gmae}}
}
=
\frac{
s_b^{\mathrm{gmae}}
}{
s_{b'}^{\mathrm{gmae}}
},
\label{eq:batch-calibration-relative-scale}
\end{equation}
preserving the relative uncertainty-aware scales across prompts.
\section{Detailed Implementation of GMAE Instantiation}
\label{app:gmae-instantiations}

This section details the three GMAE instantiations introduced in
\textcolor{deepred}{Section~\ref{sec:gmae-instantiations}}.
All three use the binary reward
\begin{equation}
\mathcal R_{\mathrm{bin}}(\mathbf y_i\mid z)
=
\mathbbm{1}\!\left[
\mathcal A(\mathbf y_i)=z
\right],
\label{eq:app-gmae-binary-reward}
\end{equation}
where \(z\) is the reward reference constructed by the corresponding
self-rewarding rule.

Given \(K\) reward realizations
\(\{r_i^k\}_{k=1}^{K}\) for response \(\mathbf y_i\), its empirical
reward distribution is
\begin{equation}
p_i
=
\frac{1}{K}\sum_{k=1}^{K}r_i^k,
\qquad
\widehat P_{r_i}
=
p_i\delta_1+(1-p_i)\delta_0.
\label{eq:app-gmae-binary-distribution}
\end{equation}
For \(\mathbf v=(v_1,\ldots,v_G)\in\{0,1\}^{G}\), define
\begin{equation}
w_{\mathbf p}(\mathbf v)
=
\prod_{j=1}^{G}
p_j^{v_j}(1-p_j)^{1-v_j},
\qquad
h_i(\mathbf v)
=
\begin{cases}
\dfrac{v_i-\mu_{\mathbf v}}{\sigma_{\mathbf v}},
& \sigma_{\mathbf v}>0,\\[6pt]
0,
& \sigma_{\mathbf v}=0,
\end{cases}
\label{eq:app-gmae-joint-weight}
\end{equation}
where
\begin{equation}
\mu_{\mathbf v}
=
\frac{1}{G}\sum_{j=1}^{G}v_j,
\qquad
\sigma_{\mathbf v}
=
\sqrt{
\frac{1}{G}
\sum_{j=1}^{G}
(v_j-\mu_{\mathbf v})^2
}.
\end{equation}
The marginalized advantage can therefore be computed exactly as
\begin{equation}
A_i^{\mathrm{gmae}}
=
\sum_{\mathbf v\in\{0,1\}^{G}}
w_{\mathbf p}(\mathbf v)h_i(\mathbf v).
\label{eq:app-gmae-binary-advantage}
\end{equation}
The first reward realization is produced by the original group used by
the corresponding base method and provides
\(A_i^{\mathrm{base}}\) for calibration.
After processing all prompts in a batch, we apply the batch-level
calibration in
\textcolor{deepred}{Eq.\ref{eq:gmae-calibration}}.

Each instantiation requires \(G'\) auxiliary rollouts beyond its base
method. We therefore focus below on the additional operations performed
after rollout generation.

\RestyleAlgo{ruled}

\subsection{GMAE$_0$: Standard Voting}

GMAE$_0$ instantiates GMAE with the majority-voting rule used by
TTRL~\colorcitep{zuo2026ttrl}.
For any response group \(\mathcal B\), its reward reference is
\begin{equation}
f_0(\mathcal B)
=
\arg\max_a
\sum_{\mathbf y\in\mathcal B}
\mathbbm{1}\!\left[
\mathcal A(\mathbf y)=a
\right].
\label{eq:app-gmae-zero-reference}
\end{equation}

The original rollout group is used as the first realization.
For each focal response, the remaining \(K-1\) contexts are sampled
uniformly without replacement from the shared candidate pool.

\begin{algorithm}[t]
\caption{GMAE$_0$: Standard Voting}
\label{alg:gmae-zero}
\KwIn{Prompt \(\mathbf x\), policy \(\pi_{\theta_{\mathrm{old}}}\),
group sizes \(G,G'\), and number of contexts \(K\)}
\KwOut{Advantages \(\{A_i^{\mathrm{gmae}}\}_{i=1}^{G}\)}

Sample
\(\mathcal C=\{\mathbf y_1,\ldots,\mathbf y_{G+G'}\}\)
from \(\pi_{\theta_{\mathrm{old}}}(\cdot\mid\mathbf x)\)\;

Set
\(\mathcal G^{\otimes G}=\{\mathbf y_1,\ldots,\mathbf y_G\}\)
and \(z^1=f_0(\mathcal G^{\otimes G})\)\;

\For{\(i=1,\ldots,G\)}{
    Set
    \(r_i^1=\mathcal R_{\mathrm{bin}}(\mathbf y_i\mid z^1)\)\;

    Sample \(K-1\) distinct contexts
    \(\{\mathcal S_i^k\}_{k=2}^{K}\) from
    \(\mathcal C\setminus\{\mathbf y_i\}\),
    with \(|\mathcal S_i^k|=G-1\)\;

    \For{\(k=2,\ldots,K\)}{
        Set
        \(\mathcal B_i^k=\{\mathbf y_i\}\cup\mathcal S_i^k\)\;

        Compute
        \(z_i^k=f_0(\mathcal B_i^k)\)\;

        Set
        \(r_i^k
        =\mathcal R_{\mathrm{bin}}(\mathbf y_i\mid z_i^k)\)\;
    }

    Compute
    \(p_i=K^{-1}\sum_{k=1}^{K}r_i^k\)\;
}

Compute \(\{A_i^{\mathrm{gmae}}\}_{i=1}^{G}\) using
\textcolor{deepred}{Eq.\ref{eq:app-gmae-binary-advantage}}\;

\Return
\(\{A_i^{\mathrm{gmae}}\}_{i=1}^{G}\)\;
\end{algorithm}

\paragraph{Additional Computation.}
Compared with TTRL, GMAE$_0$ evaluates \(K-1\) additional group
contexts for each focal response.
This adds \(G(K-1)\) majority-voting operations and the same number of
binary reward evaluations.
Since each vote scans \(G\) responses, these operations require
\(\mathcal O(G^2(K-1))\) scalar computation.
Exact distribution-based advantage estimation further enumerates at
most \(2^G\) joint binary reward realizations, requiring
\(\mathcal O(G2^G)\) scalar computation.
These operations require neither additional model generation nor
backpropagation.

\subsection{GMAE$_{\mathrm{CR}}$: Co-Refinement}

GMAE$_{\mathrm{CR}}$ follows Co-Rewarding
~\colorcitep{zhang2026corewarding}.
At training step \(t\), its co-teacher parameters are updated as
\begin{equation}
\widetilde{\theta}_t
=
\alpha_t\widetilde{\theta}_{t-1}
+
(1-\alpha_t)\theta_{\mathrm{old},t},
\label{eq:app-gmae-cr-teacher}
\end{equation}
where \(\alpha_t\in[0,1]\) is the teacher update coefficient.
For a co-teacher response group \(\widetilde{\mathcal B}\), the reward
reference is
\begin{equation}
f_{\mathrm{CR}}(\widetilde{\mathcal B})
=
\arg\max_a
\sum_{\widetilde{\mathbf y}\in\widetilde{\mathcal B}}
\mathbbm{1}\!\left[
\mathcal A(\widetilde{\mathbf y})=a
\right].
\label{eq:app-gmae-cr-reference}
\end{equation}

The reward reference depends only on the co-teacher group.
Therefore, the same \(K\) reward references can be reused for all
\(G\) main responses.

\begin{algorithm}[t]
\caption{GMAE$_{\mathrm{CR}}$: Co-Refinement}
\label{alg:gmae-cr}
\KwIn{Prompt \(\mathbf x\), parameters
\(\theta_{\mathrm{old},t}\) and \(\widetilde{\theta}_{t-1}\),
coefficient \(\alpha_t\), group sizes \(G,G'\), and number of contexts
\(K\)}
\KwOut{Advantages \(\{A_i^{\mathrm{gmae}}\}_{i=1}^{G}\)}

Update \(\widetilde{\theta}_t\) using
\textcolor{deepred}{Eq.\ref{eq:app-gmae-cr-teacher}}\;

Sample
\(\mathcal G^{\otimes G}=\{\mathbf y_i\}_{i=1}^{G}\)
from \(\pi_{\theta_{\mathrm{old},t}}(\cdot\mid\mathbf x)\)\;

Sample
\(\widetilde{\mathcal C}
=\{\widetilde{\mathbf y}_j\}_{j=1}^{G+G'}\)
from \(\pi_{\widetilde{\theta}_t}(\cdot\mid\mathbf x)\)\;

Set
\(\widetilde{\mathcal B}^{1}
=\{\widetilde{\mathbf y}_j\}_{j=1}^{G}\)\;

Sample \(K-1\) distinct groups
\(\{\widetilde{\mathcal B}^{k}\}_{k=2}^{K}\)
from \(\widetilde{\mathcal C}\), with
\(|\widetilde{\mathcal B}^{k}|=G\)\;

\For{\(k=1,\ldots,K\)}{
    Compute
    \(z_{\mathrm{CR}}^k
    =f_{\mathrm{CR}}(\widetilde{\mathcal B}^{k})\)\;
}

\For{\(i=1,\ldots,G\)}{
    \For{\(k=1,\ldots,K\)}{
        Set
        \(r_i^k
        =
        \mathcal R_{\mathrm{bin}}
        (\mathbf y_i\mid z_{\mathrm{CR}}^k)\)\;
    }

    Compute
    \(p_i=K^{-1}\sum_{k=1}^{K}r_i^k\)\;
}

Compute \(\{A_i^{\mathrm{gmae}}\}_{i=1}^{G}\) using
\textcolor{deepred}{Eq.\ref{eq:app-gmae-binary-advantage}}\;

\Return
\(\{A_i^{\mathrm{gmae}}\}_{i=1}^{G}\)\;
\end{algorithm}

\paragraph{Additional Computation.}
The co-teacher update is inherited from Co-Rewarding and therefore
introduces no GMAE-specific computation.
Compared with Co-Rewarding, GMAE$_{\mathrm{CR}}$ performs \(K-1\)
additional majority votes over groups of \(G\) co-teacher responses.
Because each reward reference is shared across all main responses, it
also adds \(G(K-1)\) binary reward evaluations.
Together with exact advantage estimation, the additional scalar
complexity is
\begin{equation}
\mathcal O\!\left(
G(K-1)+G2^G
\right).
\end{equation}

\subsection{GMAE$_{\mathrm{SR}}$: Self-Refinement}

GMAE$_{\mathrm{SR}}$ follows the self-reflective reward construction of
SR-TTRL~\colorcitep{wu2026srttrl}.
Instead of independently performing self-reflection for every group
context, it first ranks all candidate answers in the shared pool and
then reuses this ranking across contexts.

Given the shared candidate pool
\(\mathcal C=\{\mathbf y_1,\ldots,\mathbf y_{G+G'}\}\), let
\begin{equation}
\mathcal U_{\mathcal C}
=
\left\{
\mathcal A(\mathbf y_j)
\right\}_{j=1}^{G+G'}
\label{eq:app-gmae-sr-answer-set}
\end{equation}
denote its set of distinct answers.
For each \(a\in\mathcal U_{\mathcal C}\), GMAE$_{\mathrm{SR}}$ selects
a representative response and summarizes its reasoning:
\begin{equation}
\mathbf s_a
=
\operatorname{Summ}_{\pi_{\theta_{\mathrm{old}}}}
\left(
\mathbf x,
\operatorname{Rep}
\left\{
\mathbf y_j\in\mathcal C
\mid
\mathcal A(\mathbf y_j)=a
\right\}
\right).
\label{eq:app-gmae-sr-summary}
\end{equation}
It then performs self-reflection once over all answer-summary pairs,
producing a ranking \(\succ_{\mathrm{SR}}\) over
\(\mathcal U_{\mathcal C}\):
\begin{equation}
\succ_{\mathrm{SR}}
=
\operatorname{Reflect}_{\pi_{\theta_{\mathrm{old}}}}
\left(
\mathbf x,
\left\{
(a,\mathbf s_a)
\right\}_{a\in\mathcal U_{\mathcal C}}
\right).
\label{eq:app-gmae-sr-ranking}
\end{equation}
Here, \(a\succ_{\mathrm{SR}}a'\) means that \(a\) is ranked above
\(a'\) by self-reflection.
The representative-selection rule, summary prompt, and reflection
prompt follow the original SR-TTRL implementation.

For the \(k\)-th group context of response \(\mathbf y_i\), define
\begin{equation}
\mathcal B_i^k
=
\{\mathbf y_i\}\cup\mathcal S_i^k,
\qquad
\mathcal U_i^k
=
\left\{
\mathcal A(\mathbf y)
\mid
\mathbf y\in\mathcal B_i^k
\right\}.
\label{eq:app-gmae-sr-context-answers}
\end{equation}
Its reward reference is the highest-ranked answer present in this
context:
\begin{equation}
z_{\mathrm{SR},i}^k
=
\max_{\succ_{\mathrm{SR}}}\mathcal U_i^k,
\qquad
r_i^k
=
\mathcal R_{\mathrm{bin}}
\left(
\mathbf y_i
\mid
z_{\mathrm{SR},i}^k
\right).
\label{eq:app-gmae-sr-context-reference}
\end{equation}

\begin{algorithm}[t]
\caption{GMAE$_{\mathrm{SR}}$: Self-Refinement}
\label{alg:gmae-sr}
\KwIn{Prompt \(\mathbf x\), policy \(\pi_{\theta_{\mathrm{old}}}\),
group sizes \(G,G'\), and number of contexts \(K\)}
\KwOut{Advantages \(\{A_i^{\mathrm{gmae}}\}_{i=1}^{G}\)}

Sample
\(\mathcal C=\{\mathbf y_1,\ldots,\mathbf y_{G+G'}\}\)
from \(\pi_{\theta_{\mathrm{old}}}(\cdot\mid\mathbf x)\)\;

Set
\(\mathcal G^{\otimes G}
=\{\mathbf y_1,\ldots,\mathbf y_G\}\)\;

Construct the distinct-answer set
\(\mathcal U_{\mathcal C}\)\;

\ForEach{\(a\in\mathcal U_{\mathcal C}\)}{
    Select
    \(\mathbf y_a^{\mathrm{rep}}
    =
    \operatorname{Rep}
    \{\mathbf y_j\in\mathcal C:
    \mathcal A(\mathbf y_j)=a\}\)\;

    Compute
    \(\mathbf s_a
    =
    \operatorname{Summ}_{\pi_{\theta_{\mathrm{old}}}}
    (\mathbf x,\mathbf y_a^{\mathrm{rep}})\)\;
}

Compute the global answer ranking
\(\succ_{\mathrm{SR}}\) using
\textcolor{deepred}{Eq.\ref{eq:app-gmae-sr-ranking}}\;

\For{\(i=1,\ldots,G\)}{
    Set
    \(\mathcal S_i^1
    =
    \mathcal G^{\otimes G}\setminus\{\mathbf y_i\}\)\;

    Sample \(K-1\) additional distinct contexts
    \(\{\mathcal S_i^k\}_{k=2}^{K}\) from
    \(\mathcal C\setminus\{\mathbf y_i\}\),
    with \(|\mathcal S_i^k|=G-1\)\;

    \For{\(k=1,\ldots,K\)}{
        Set
        \(\mathcal B_i^k
        =
        \{\mathbf y_i\}\cup\mathcal S_i^k\)\;

        Extract the answers
        \(\mathcal U_i^k
        =
        \{\mathcal A(\mathbf y):
        \mathbf y\in\mathcal B_i^k\}\)\;

        Select
        \(z_{\mathrm{SR},i}^k
        =
        \max_{\succ_{\mathrm{SR}}}\mathcal U_i^k\)\;

        Set
        \(r_i^k
        =
        \mathcal R_{\mathrm{bin}}
        (\mathbf y_i\mid z_{\mathrm{SR},i}^k)\)\;
    }

    Compute
    \(p_i=K^{-1}\sum_{k=1}^{K}r_i^k\)\;
}

Compute \(\{A_i^{\mathrm{gmae}}\}_{i=1}^{G}\) using
\textcolor{deepred}{Eq.\ref{eq:app-gmae-binary-advantage}}\;

\Return
\(\{A_i^{\mathrm{gmae}}\}_{i=1}^{G}\)\;
\end{algorithm}

\paragraph{Additional Computation.}
Let \(U_{\mathrm{base}}\) and \(U_{\mathrm{pool}}\) denote the numbers
of distinct answers in the original rollout group and the shared
candidate pool, respectively.
SR-TTRL summarizes \(U_{\mathrm{base}}\) answer classes and performs
one reflection, whereas GMAE$_{\mathrm{SR}}$ summarizes
\(U_{\mathrm{pool}}\) classes and also performs only one reflection.
Therefore, GMAE$_{\mathrm{SR}}$ adds
\(U_{\mathrm{pool}}-U_{\mathrm{base}}\leq G'\) summary operations and
does not increase the number of reflection operations, although the
single reflection covers a larger candidate set.

Once the global ranking is obtained, each additional context requires
only extracting its answers and selecting the highest-ranked one.
Compared with SR-TTRL, this introduces \(G(K-1)\) context-specific
ranking lookups and binary reward evaluations.
A direct implementation scans at most \(G\) responses per context,
giving \(\mathcal O(G^2(K-1))\) scalar computation.
Exact distribution-based advantage estimation further requires
\(\mathcal O(G2^G)\) scalar operations.
Thus, GMAE$_{\mathrm{SR}}$ reuses a single self-reflective ranking
across all contexts rather than repeating the expensive reflection
procedure \(K\) times.
\section{Full Scalability Analysis Results}
\label{app:scalability-results}


\begin{table}[H]
\caption{
\textbf{Scalability on RL backbone.}
Benchmark-wise \textsc{Mean@16} results on
\texttt{Qwen3-8B-Base} using GSPO as the RL backbone.
}
\label{tab:scalability-gspo}
\vspace{-0.08in}
\centering
\footnotesize
\renewcommand{\arraystretch}{1.1}
\setlength{\tabcolsep}{1.5mm}

\resizebox{\textwidth}{!}{
\begin{tabular}{lccccccccc}
\toprule

\multirow{2}{*}{\textbf{Method}}
& \multicolumn{5}{c}{\textbf{Mathematics}}
& \textbf{Science}
& \textbf{Knowledge}
& \textbf{Coding}
& \multirow{2}{*}{\textbf{\textit{Average}}}
\\

\cmidrule(lr){2-6}
\cmidrule(lr){7-7}
\cmidrule(lr){8-8}
\cmidrule(lr){9-9}

& \textbf{MATH500}
& \textbf{AMC}
& \textbf{AIME24}
& \textbf{AIME25}
& \textbf{AIME26}
& \textbf{GPQA}
& \textbf{MMLU-Pro}
& \textbf{LiveCode}
&
\\

\midrule

TTRL
& 78.8 & 45.3 & 14.4 & 11.3 & 7.2 & 40.3 & 66.6 & 22.3 & 35.8 \\

Self-Harmony
& 78.4 & 43.2 & 15.2 & 11.0 & 6.7 & 40.9 & \textbf{70.8} & 22.8 & 36.1 \\

Co-Reward
& 80.3 & 47.7 & 15.5 & 12.7 & 7.9 & 42.0 & 68.7 & 23.2 & 37.3 \\

SR-TTRL
& 79.7 & 47.1 & 14.7 & 12.6 & 8.0 & 41.4 & 67.6 & 22.9 & 36.8 \\

\rowcolor{gray!20}
\textbf{GMAE$_0$ (Ours)}
& 81.9 & 48.7 & 17.3 & 14.0 & 9.0 & 42.1 & 67.9 & 24.2 & 38.1 \\

\rowcolor{gray!20}
\textbf{GMAE$_{\text{CR}}$ (Ours)}
& \textbf{83.3}
& \textbf{50.0}
& \textbf{17.6}
& \textbf{15.0}
& \textbf{9.7}
& \textbf{43.8}
& 69.7
& \textbf{25.0}
& \textbf{39.3} \\

\rowcolor{gray!20}
\textbf{GMAE$_{\text{SR}}$ (Ours)}
& 82.4 & 49.8 & 17.5 & 14.5 & 9.1 & 43.2 & 68.5 & 24.8 & 38.7 \\

\bottomrule
\end{tabular}
}
\end{table}


\begin{table}[H]
\caption{
\textbf{Scalability on RL backbone.}
Benchmark-wise \textsc{Mean@16} results on
\texttt{Qwen3-8B-Base} using REINFORCE++ as the RL backbone.
}
\label{tab:scalability-reinforce}
\vspace{-0.08in}
\centering
\footnotesize
\renewcommand{\arraystretch}{1.1}
\setlength{\tabcolsep}{1.5mm}

\resizebox{\textwidth}{!}{
\begin{tabular}{lccccccccc}
\toprule

\multirow{2}{*}{\textbf{Method}}
& \multicolumn{5}{c}{\textbf{Mathematics}}
& \textbf{Science}
& \textbf{Knowledge}
& \textbf{Coding}
& \multirow{2}{*}{\textbf{\textit{Average}}}
\\

\cmidrule(lr){2-6}
\cmidrule(lr){7-7}
\cmidrule(lr){8-8}
\cmidrule(lr){9-9}

& \textbf{MATH500}
& \textbf{AMC}
& \textbf{AIME24}
& \textbf{AIME25}
& \textbf{AIME26}
& \textbf{GPQA}
& \textbf{MMLU-Pro}
& \textbf{LiveCode}
&
\\

\midrule

TTRL
& 75.9 & 40.1 & 10.7 & 9.7 & 5.7 & 36.8 & 64.3 & 19.4 & 32.8 \\

Self-Harmony
& 75.5 & 39.6 & 11.1 & 10.4 & 6.3 & 38.1 & \textbf{68.3} & 19.6 & 33.6 \\

Co-Reward
& 76.7 & 41.5 & 11.9 & 11.5 & 6.6 & 37.7 & 64.5 & 20.4 & 33.9 \\

SR-TTRL
& 76.9 & 40.8 & 11.7 & 10.8 & 6.1 & 37.4 & 65.0 & 20.1 & 33.6 \\

\rowcolor{gray!20}
\textbf{GMAE$_0$ (Ours)}
& 79.8 & 45.5 & 13.8 & 12.9 & 9.3 & 40.2 & 66.7 & 22.8 & 36.4 \\

\rowcolor{gray!20}
\textbf{GMAE$_{\text{CR}}$ (Ours)}
& \textbf{80.5}
& \textbf{47.8}
& \textbf{14.9}
& \textbf{14.2}
& \textbf{10.2}
& \textbf{41.0}
& 66.8
& \textbf{24.3}
& \textbf{37.5} \\

\rowcolor{gray!20}
\textbf{GMAE$_{\text{SR}}$ (Ours)}
& 80.0 & 47.1 & 14.4 & 13.7 & 9.2 & \textbf{41.0} & 67.4 & 23.7 & 37.1 \\

\bottomrule
\end{tabular}
}
\end{table}


\begin{table}[H]
\caption{
\textbf{Scalability on training dataset.} \textsc{Mean@16} results on
\texttt{Qwen3-8B-Base} trained with the \texttt{Open-RS} dataset.
}
\label{tab:scalability-openrs}
\vspace{-0.08in}
\centering
\footnotesize
\renewcommand{\arraystretch}{1.1}
\setlength{\tabcolsep}{1.5mm}

\resizebox{\textwidth}{!}{
\begin{tabular}{lccccccccc}
\toprule

\multirow{2}{*}{\textbf{Method}}
& \multicolumn{5}{c}{\textbf{Mathematics}}
& \textbf{Science}
& \textbf{Knowledge}
& \textbf{Coding}
& \multirow{2}{*}{\textbf{\textit{Average}}}
\\

\cmidrule(lr){2-6}
\cmidrule(lr){7-7}
\cmidrule(lr){8-8}
\cmidrule(lr){9-9}

& \textbf{MATH500}
& \textbf{AMC}
& \textbf{AIME24}
& \textbf{AIME25}
& \textbf{AIME26}
& \textbf{GPQA}
& \textbf{MMLU-Pro}
& \textbf{LiveCode}
&
\\

\midrule

TTRL
& 76.4 & 42.7 & 13.5 & 10.4 & 7.6 & 38.8 & 62.1 & 20.3 & 34.0 \\

Self-Harmony
& 75.6 & 41.6 & 13.7 & 10.8 & 7.2 & 39.3 & \textbf{69.3} & 21.4 & 34.9 \\

Co-Reward
& 77.3 & 43.2 & 14.1 & 11.6 & 7.8 & 40.4 & 64.9 & 22.7 & 35.3 \\

SR-TTRL
& 77.0 & 43.5 & 13.8 & 11.4 & 7.8 & 40.5 & 65.7 & 22.2 & 35.2 \\

\rowcolor{gray!20}
\textbf{GMAE$_0$ (Ours)}
& 77.8 & 44.0 & 14.9 & 12.3 & 8.3 & 41.2 & 67.9 & 22.9 & 36.2 \\

\rowcolor{gray!20}
\textbf{GMAE$_{\text{CR}}$ (Ours)}
& \textbf{79.3}
& \textbf{46.5}
& \textbf{16.3}
& \textbf{13.7}
& \textbf{9.0}
& \textbf{42.3}
& 68.5
& \textbf{23.1}
& \textbf{37.3} \\

\rowcolor{gray!20}
\textbf{GMAE$_{\text{SR}}$ (Ours)}
& 78.6 & 46.1 & 16.2 & 13.0 & 8.8 & 40.3 & 67.2 & 22.9 & 36.6 \\

\bottomrule
\end{tabular}
}
\end{table}


\begin{table}[H]
\caption{
\textbf{Scalability on training dataset.} \textsc{Mean@16} results on
\texttt{Qwen3-8B-Base} trained with the \texttt{MATH-8K} dataset.
}
\label{tab:scalability-math8k}
\vspace{-0.08in}
\centering
\footnotesize
\renewcommand{\arraystretch}{1.1}
\setlength{\tabcolsep}{1.5mm}

\resizebox{\textwidth}{!}{
\begin{tabular}{lccccccccc}
\toprule

\multirow{2}{*}{\textbf{Method}}
& \multicolumn{5}{c}{\textbf{Mathematics}}
& \textbf{Science}
& \textbf{Knowledge}
& \textbf{Coding}
& \multirow{2}{*}{\textbf{\textit{Average}}}
\\

\cmidrule(lr){2-6}
\cmidrule(lr){7-7}
\cmidrule(lr){8-8}
\cmidrule(lr){9-9}

& \textbf{MATH500}
& \textbf{AMC}
& \textbf{AIME24}
& \textbf{AIME25}
& \textbf{AIME26}
& \textbf{GPQA}
& \textbf{MMLU-Pro}
& \textbf{LiveCode}
&
\\

\midrule

TTRL
& 75.5 & 41.2 & 12.3 & 9.5 & 6.7 & 37.5 & 63.2 & 19.5 & 33.2 \\

Self-Harmony
& 76.4 & 41.5 & 12.0 & 9.9 & 6.5 & 37.3 & 61.3 & 19.4 & 33.0 \\

Co-Reward
& 77.0 & 42.3 & 13.6 & 8.9 & 7.1 & 38.2 & \textbf{66.0} & 21.2 & 34.3 \\

SR-TTRL
& 77.1 & 42.3 & 13.6 & 10.4 & 7.8 & 38.6 & 63.9 & 21.0 & 34.3 \\

\rowcolor{gray!20}
\textbf{GMAE$_0$ (Ours)}
& 79.1 & 43.7 & 15.0 & 11.1 & 8.3 & 40.1 & 65.1 & 21.2 & 35.5 \\

\rowcolor{gray!20}
\textbf{GMAE$_{\text{CR}}$ (Ours)}
& 80.3 & \textbf{44.0} & \textbf{15.7} & 12.0 & 9.0 & 41.0 & 64.4 & 22.9 & 36.2 \\

\rowcolor{gray!20}
\textbf{GMAE$_{\text{SR}}$ (Ours)}
& \textbf{80.8} & 43.5 & 15.4 & \textbf{12.3} & \textbf{9.5} & \textbf{41.2} & 66.1 & \textbf{23.2} & \textbf{36.5} \\

\bottomrule
\end{tabular}
}
\end{table}

\end{document}